\documentclass{article}

\PassOptionsToPackage{numbers, compress}{natbib}

 \usepackage[preprint]{neurips_2026}

\usepackage[utf8]{inputenc} 
\usepackage[T1]{fontenc}    
\usepackage{hyperref}       
\usepackage{url}            
\usepackage{booktabs}       
\usepackage{amsfonts}       
\usepackage{nicefrac}       
\usepackage{microtype}      
\usepackage[dvipsnames]{xcolor}
\usepackage{amsmath}
\usepackage{amssymb}
\usepackage{mathtools}
\usepackage{amsthm}
\usepackage{algorithm}
\usepackage{algorithmicx}
\usepackage{underscore}
\usepackage{algpseudocode}
\usepackage{makecell}
\usepackage{colortbl}
\usepackage{multirow}
\usepackage{graphicx}
\usepackage{wrapfig}
\usepackage{subcaption}
\usepackage{textcomp}
\usepackage{siunitx}
\usepackage{soul}

\newcommand{\ttmath}{%
  \everymath{%
    {\scriptstyle\mathtt{}}%
    {\scriptscriptstyle\mathtt{}}%
    \mathtt{\xdef\tmp{\fam\the\fam\relax}\aftergroup\tmp}}}
\newcommand{\val}[1]{\bgroup\ttmath\texttt{#1}\egroup}
\newcommand{\ttp}[2]{#1 $\cdot$ 10$^{#2}$}
\newcommand{\ttpstd}[3]{(#1 {\tiny$\pm$#2}) $\cdot$ 10$^{#3}$}
\newcommand{\mest}[2]{#1 {\tiny $\pm$#2}}

\setul{0.2ex}{0.2ex}
\setulcolor{NavyBlue}
\definecolor{highlightblue}{RGB}{77, 148, 255}

\title{Native Extrapolation Awareness in Flow-Based Conditional Generation}

{
    \small
\author{%
  Constantinos Tsakonas \\
  Inria, Université de Lorraine, CNRS, Loria \\
  Nancy, France\\
  \texttt{konstantinos.tsakonas@inria.fr} \\
  \And
  Serena Ivaldi \\
  Inria, Université de Lorraine, CNRS, Loria \\
  Nancy, France\\
  \texttt{serena.ivaldi@inria.fr} \\
  \And
  Jean-Baptiste Mouret \\
  Inria, Université de Lorraine, CNRS, Loria \\
  Nancy, France\\
  \texttt{jean-baptiste.mouret@inria.fr} \\
}
}

\theoremstyle{plain}
\newtheorem{theorem}{Theorem}[section]
\newtheorem{proposition}[theorem]{Proposition}

\theoremstyle{definition}

\theoremstyle{remark}
\newtheorem{remark}[theorem]{Remark}
\newcommand{\algo}{Diverging Flows}
\newcommand{\algotext}[1]{\texttt{#1}}
\newcommand{\algosmall}{DiFlo}
\newcommand{\algoours}[1]{\algotext{\color{NavyBlue}{{\algosmall} #1}}}

\begin{document}

\maketitle

\begin{abstract}
The ability of Flow Matching (FM) to model complex conditional distributions has established it as the state-of-the-art for prediction tasks (e.g., robotics, weather forecasting). However, deployment in safety-critical settings is hindered by a critical extrapolation hazard: driven by smoothness biases, flow models yield plausible outputs even for off-manifold conditions, resulting in silent failures indistinguishable from valid predictions. In this work, we introduce \emph{Diverging Flows}, a novel approach that enables a single model to simultaneously perform conditional generation and native extrapolation detection by structurally enforcing inefficient transport for off-manifold inputs. We evaluate our method on synthetic manifolds, cross-domain style transfer, and weather temperature forecasting, demonstrating that it achieves effective detection of extrapolations without compromising predictive fidelity or inference latency. These results establish \emph{Diverging Flows} as a robust solution for trustworthy flow models, paving the way for reliable deployment in domains such as medicine, robotics, and climate science.
\end{abstract}

\section{Introduction}\label{sec:introduction}

Deep neural networks are increasingly outperforming traditional predictors in safety-sensitive domains, for
instance replacing Model Predictive Control (MPC) in robotics~\citep{levine2016end, miki2022learning,
kaufmann2023champion, hwangbo2019learning, akkaya2019solving} or physics-based modeling in weather
forecasting~\citep{lam2023graphcast, bi2023accurate, kochkov2024neural}. Their deployment in production is
nonetheless hindered by \emph{silent extrapolation}: under anomalous or out-of-distribution (OOD) conditions,
deep models do not diverge or abstain but produce predictions that remain structured and downstream-valid,
masking failure as nominal behavior. The consequences can be catastrophic. In the fatal 2018 Tempe
accident~\citep{ntsb2019uber}, the perception stack repeatedly oscillated between different classification labels for a
cyclist walking a bike, and the vehicle continued on its trajectory rather than triggering an emergency stop.
In climate science, FourCastNet trained without Category 3--5 tropical cyclones cannot forecast unseen
Category 5 events: it reverts to its learned distribution and predicts mild conditions while a catastrophe
approaches~\citep{sun2025grayswan, pasche2025validating}, instead of flagging the input as outside its
competence and deferring to a physics-based model.

These failures directly follow from how these models are trained. 
Standard supervised and generative models are designed to interpolate within
the training distribution and offer no guarantees beyond it; a model must therefore not be trusted on inputs for
which it was not trained~\citep{extrapolation_sage}, making extrapolation detection a prerequisite for accountable
use~\citep{lahoti2023responsible}. This requirement is increasingly codified in regulation: the EU AI Act mandates
uncertainty estimation for compliance and transparency in high-risk
systems~\citep{valdenegro2024dilemma, herrera2025responsible}.

\begin{wrapfigure}[24]{r}{0.44\textwidth}
     \centering
     \includegraphics[width=0.75\linewidth]{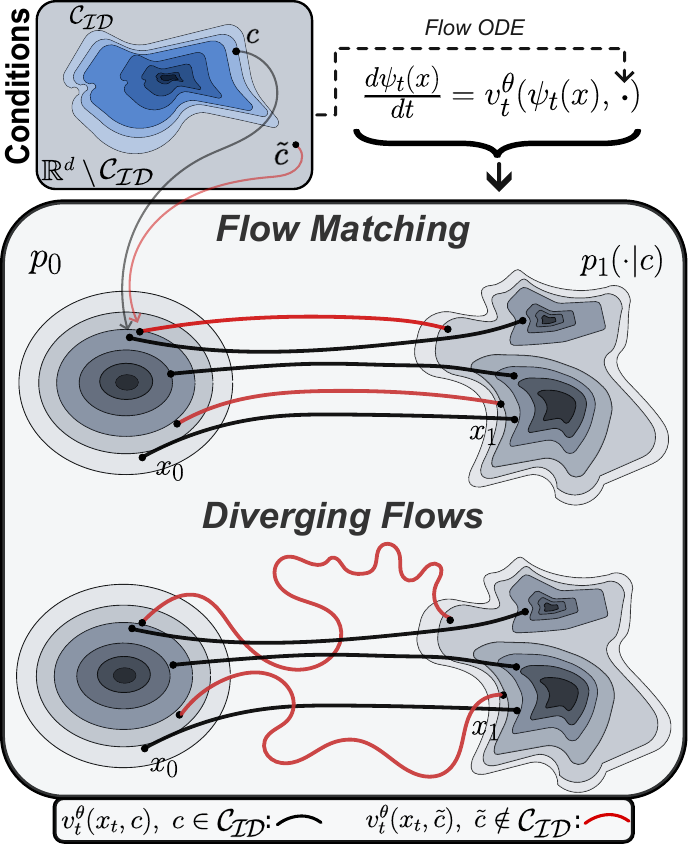}
    \captionsetup{font=small}
    \caption{\textbf{Conceptual Overview.} (Top) Standard FM produces a smooth flow even for off-manifold conditions (red), silently converging to a plausible output. (Bottom) {\algo} structurally enforces transport inefficiency: valid conditions ($c$) induce near-optimal geodesics, while off-manifold queries ($\tilde{c}$) trigger detectable divergence.}
     \label{fig:concept_fig}
\end{wrapfigure}
One might hope that the recent shift toward probabilistic regression would mitigate this issue. Because complex
physical dynamics are inherently stochastic, deterministic predictors have been largely superseded by conditional
generative models, specifically Diffusion~\citep{sohl2015deep, ho2020denoising} and Flow Matching
(FM)~\citep{lipman2022flow, tong2023improving}, which now set the state of the art for these tasks. Optimal
Transport FM (OT-FM) is particularly favored because regressing the model onto straight conditional paths between
source and target distributions yields efficient, numerically stable training targets. OT-FM and its diffusion
counterparts are consequently displacing prior methods in robotics~\citep{chi2023diffusion, black2024pi_0,
rouxel2024flow, rouxel2025extremum}, weather forecasting~\citep{chen2024flow, price2025probabilistic}, and
structure-based drug design~\citep{corso2023diffdock}, among others.

Modeling the conditional distribution, however, does not address whether the conditioning input itself is one the
model should be predicting on. The structural properties that make FM efficient make it a near-perfect silent
extrapolator. Driven by inductive biases toward spectral smoothness~\citep{rahaman2019spectral}, the learned
neural vector fields construct smooth paths between almost any input and output. Queried with physically
impossible or semantically invalid off-manifold conditions, the model neither degrades nor diverges; it transports
the source distribution to a plausible target, by design. While useful in creative media
synthesis~\citep{mariani2024multisource, guerreiro2024layoutflow, kim2024generalized, ki2025float}, this behavior
is not acceptable in safety-sensitive regression.

Standard post-hoc safety checks fail to resolve this silent failure. The ``likelihood paradox'' demonstrates that
generative models often assign high likelihood to simple off-manifold inputs~\citep{nalisnick2018deep, serra2019input},
rendering statistical density an unreliable proxy for validity. Alternatively, training a separate classifier or one-class detector~\citep{yajie2023survey, chalapathy2019deep} doubles inference cost and decouples detection from the primary dynamics, while outlier exposure~\citep{hendrycks2018deep} relies on a representative set of ``unknowns'' that is rarely available.

Our starting point is a property of continuous-time conditional modeling: the target probability path
$p_t(x|c)$ only needs to accurately capture transport dynamics for in-distribution conditions. Because the vector field
is largely unconstrained outside the training support, it presents a unique opportunity to intrinsically encode
information about the condition distribution. The same property that makes OT-FM efficient---straight conditional
paths---can be repurposed as a natural signal of validity. We show that an FM model can be trained to exhibit a
\emph{geometric phase transition}, where valid conditions yield near-optimal paths while off-manifold queries
structurally break the vector field. This demonstrates that detection can emerge directly from the geometry of the
generative process, letting the model signal its limits through path turbulence rather than relying on an external
statistical score, closely aligning with principles of conservative distribution learning~\citep{ma2021conservative,
zhengdao2024conservative, lubold2022formal}.

To practically realize this geometric phase transition, our contributions are as follows:
\textbf{(i)} We provide a theoretical characterization of silent extrapolation, formalizing how spectral smoothness
artificially enforces near-optimal transport on off-manifold conditions;
\textbf{(ii)} We introduce {\algo}, a contrastive training objective that structurally breaks this smoothness by enforcing transport inefficiency for invalid inputs;
\textbf{(iii)} We derive a geometric detection criterion that extracts an extrapolation score directly from trajectory turbulence during inference; and
\textbf{(iv)} We extensively evaluate our framework across synthetic manifolds, physical weather forecasting, and semantic style transfer, demonstrating near-perfect extrapolation awareness with negligible inference overhead.

\section{Related Work}

\subsection{Generative Models for Predictive Tasks}
In robotics, FM and Diffusion models have become a successful alternative for visuomotor control policies,
learning to map observations to precise action trajectories~\cite{chi2023diffusion, black2024pi_0,
rouxel2024flow}. In computational biology, these models have been successfully applied to protein backbone
design, regressing 3D molecular structures from sequence conditioning~\citep{jing2023alphafold, bose2023se}. Similarly,
in earth sciences and healthcare, FM and diffusion models are deployed for global weather
forecasting~\citep{chen2024flow, price2025probabilistic, andrae2025continuous, chan2024estimating}, clinical time-series
modeling~\citep{zhang2024trajectory}, as well as, general time-series forecasting~\citep{kollovieh2024flow,
kollovieh2023predict}. However, while these works demonstrate the efficacy of generative models for high-fidelity
regression, they typically prioritize in-distribution performance. They lack intrinsic mechanisms to flag off-manifold
queries, relying entirely on the model's ability to extrapolate, which poses severe risks in safety-critical deployments.

\subsection{Off-Manifold Detection with Generative Models}
While extensive literature exists for OOD detection in discriminative models~\citep{liang2017enhancing, sun2022out,
Li_2023_CVPR}, addressing this challenge in conditional generative models remains underexplored. Prior unsupervised
efforts primarily rely on likelihood criteria or reconstruction scores. However, likelihood-based approaches are not
sufficient~\citep{nalisnick2018deep, hendrycks2018deep}. Attempts to mitigate this via input complexity
corrections~\citep{serra2019input} or generative ensembles~\citep{choi2018waic} have met with limited success, often
functioning as computationally expensive post-hoc solutions than tackling the root cause.

Similarly, reconstruction-based methods using autoencoders or diffusion models~\citep{liu2023unsupervised,
graham2023denoising} detect anomalies based on restoration errors. While effective as solely OOD detector, these strategies
operate on the \textit{output} space, assessing generation quality rather than validating the conditioning input itself,
and typically serve as expensive external wrappers rather than integrated safety mechanisms.

Lastly, DiffPath~\citep{heng2024out} proposes a passive detection strategy, meaning that a second model is required for
the predictive task, based on analyzing the dynamics of the reverse
generation process. However, the applicability of this framework to our setting is restricted by two structural
differences. First, it is designed for unconditional models and has not been extended to conditional
generation. Consequently, it inherently lacks a mechanism to explicitly evaluate the validity of the conditioning input.
Second, DiffPath relies on the specific curvature of diffusion SDE trajectories. It is incompatible with FM, where
trajectories are inherently smooth, erasing the geometric signals DiffPath requires.

\section{Problem Formulation and Motivation}\label{motivation}

\subsection{Problem Formulation}

We define the task of \emph{Extrapolation-Aware Probabilistic Regression} as learning a mapping $f_\theta: \mathcal{C}
\to \mathcal{P}(\mathcal{X}) \times \mathbb{R}$ that outputs a predictive distribution and an extrapolation indicator. Thus, for an input $c$:
\begin{equation}
    f_\theta(c) = \left( p_\theta(x|c), \ S_\theta(c) \right)
    \label{eq:safety_mapping}
\end{equation}
where $p_\theta(\cdot|c)$ approximates the conditional data density $p_{data}(x|c)$, and $S_\theta(c)$ serves as an extrapolation score. To ensure safety, this tuple must satisfy two critical properties: (i) \emph{Fidelity}, where for in-distribution inputs
($c \in \mathcal{M}$), $p_\theta$ minimizes the divergence from the true conditional distribution; and (ii)
\emph{Trust}, where the score acts as a discriminator for manifold membership, effectively separating valid from invalid
inputs such that $S_\theta(\tilde{c}) \gg S_\theta(c)$ for any $\tilde{c} \notin \mathcal{M}$ and $c \in \mathcal{M}$.

\subsection{Flow Matching}

We realize the predictive component of our formulation using the FM framework~\citep{lipman2022flow}. 
While we frame our analysis in terms of probabilistic regression, this formulation is mathematically
general and naturally extends to \textbf{conditional generation tasks}. The generative process in conditional generation is governed by the flow ODE:
\begin{equation}
    \frac{d\psi_t(x)}{dt} = v_t^\theta(\psi_t(x), c),
    \label{eq:flow_ode}
\end{equation}
where $t \in [0, 1]$, and $v_t^\theta$ is a neural vector field that pushes a noise prior $p_0$ to the target density
$p_1(\cdot|c)$. This probabilistic formulation naturally handles aleatoric uncertainty arising from inherent data noise.
By sampling $N$ trajectories $\{\psi_1^{(i)}(c)\}_{i=1}^N \sim p_\theta(\cdot|c)$ generated by solving
Eq.~\ref{eq:flow_ode}, we approximate the conditional expectation to obtain a robust prediction $\hat{x}_1 =
\mathbb{E}_{p_\theta}[x_1|c]$. While this Monte Carlo integration filters stochastic noise, it fails to ensure model trustworthiness.
For an off-manifold condition $\tilde{c}$ disjoint from the training support, the FM continues to generate a
coherent, low-variance predictive distribution, yielding a plausible but extrapolated prediction $\hat{x}_1$, rendering
it insufficient for safety-sensitive applications.


\section{\algo}\label{methodology}

\subsection{The Extrapolation Problem}\label{sec:extrapolation_problem}

Let $\mathcal{X} = \mathbb{R}^d$ denote the data space and $\mathcal{C} = \mathbb{R}^k$ the conditioning space. We
assume valid conditions lie on a lower-dimensional manifold $\mathcal{M} \subset \mathcal{C}$, while its complement
$\mathcal{M}^c = \mathbb{R}^k \setminus \mathcal{M}$ represents the off-manifold regime. Optimal Transport theory establishes that the most efficient probability path follows the Euclidean geodesic $u_t^{OT}(x_t|x_1) = x_1 - x_0$. We define the \emph{Transport Energy Excess}, $\mathcal{E}(v^\theta; c)$, as the expected deviation of the learned vector field from this optimal geodesic:
\begin{equation}
    \mathcal{E}(v^\theta; c) = \mathbb{E}_{t, x_t} \left[ ||v_t^\theta(x_t, c) - u_t^{OT}(x_t|x_1)||_2^2 \right].
    \label{eq:energy_functional}
\end{equation}

Minimizing $\mathcal{E}(v^\theta; c)$ is mathematically equivalent to the standard OT-FM objective, $\mathcal{L}_\algotext{FM}$ \citep{lipman2022flow}. Ideally, a safe conditional predictor would exhibit a \emph{geometric phase transition}: minimizing $\mathcal{E}(v^\theta; c) \to 0$ for $c \in \mathcal{M}$ (yielding high fidelity) while maximizing $\mathcal{E}(v^\theta; \tilde{c}) \gg 0$ for $\tilde{c} \in \mathcal{M}^c$ (detectable divergence).

However, achieving this phase transition is fundamentally hindered by the structural regularity required of neural ODEs.
For the solution $\psi_t$ to be unique, the learned vector field must be locally Lipschitz
continuous\citep{coddington1955theory, lipman2024flow}. This mathematical constraint creates an \emph{extrapolation hazard}, which we formalize below:

\begin{proposition}[Lipschitz Continuity of Transport Energy]
\label{prop:lipschitz_energy}
Let $v_t^\theta(x, c)$ be a continuous neural vector field, uniformly bounded by $M$, that is locally Lipschitz
    continuous with respect to the conditioning variable $c$ with constant $K$. Let $\mathcal{E}(c) = \mathbb{E}_{t,
    x_t}[||v_t^\theta(x_t, c) - u^{OT}||_2^2]$ be the Transport Energy Excess of the vector field relative to a
    constant optimal transport velocity $u^{OT}$. Then, $\mathcal{E}(c)$ is locally Lipschitz continuous with a constant $L \le
    2K(M + ||u^{OT}||_2)$ (Proof in Appendix \ref{proof:lipschitz_energy}).
\end{proposition}

It is well-established that the implicit regularization of deep neural networks (e.g., spectral bias) favors the
learning of smooth, low-frequency functions \citep{rahaman2019spectral}. This inductive bias constrains the
Lipschitz constant $K$ of the vector field to remain small. As demonstrated by Proposition \ref{prop:lipschitz_energy}, a bounded $K$ mathematically forces the transport energy landscape to remain flat (i.e., bounded by a correspondingly small $L$). Because standard training explicitly drives $\mathcal{E}(v^\theta; c) \approx 0$ for valid data, this tight continuity
bound guarantees that for any proximate off-manifold condition $\tilde{c} \in \mathcal{M}^c$, the transport energy
cannot increase sufficiently to trigger an anomaly signal, since $|\mathcal{E}(v^\theta; \tilde{c}) - \mathcal{E}(v^\theta; c)| \leq L \cdot ||\tilde{c} - c||_2 \approx 0.$

This dynamic establishes the structural root of the extrapolation problem: without active geometric regularization, the
inherent smoothness of the neural vector field fundamentally masks invalid inputs. The model defaults to the simplest,
most efficient path regardless of validity, resulting in silent extrapolations that perfectly mimic valid flows, without
signaling any uncertainty.

\subsection{Active Geometric Divergence via Vector Field Margins}
\label{sec:active_geometric_divergence}

\begin{figure}[t]
    \centering
    \includegraphics[width=0.87\textwidth]{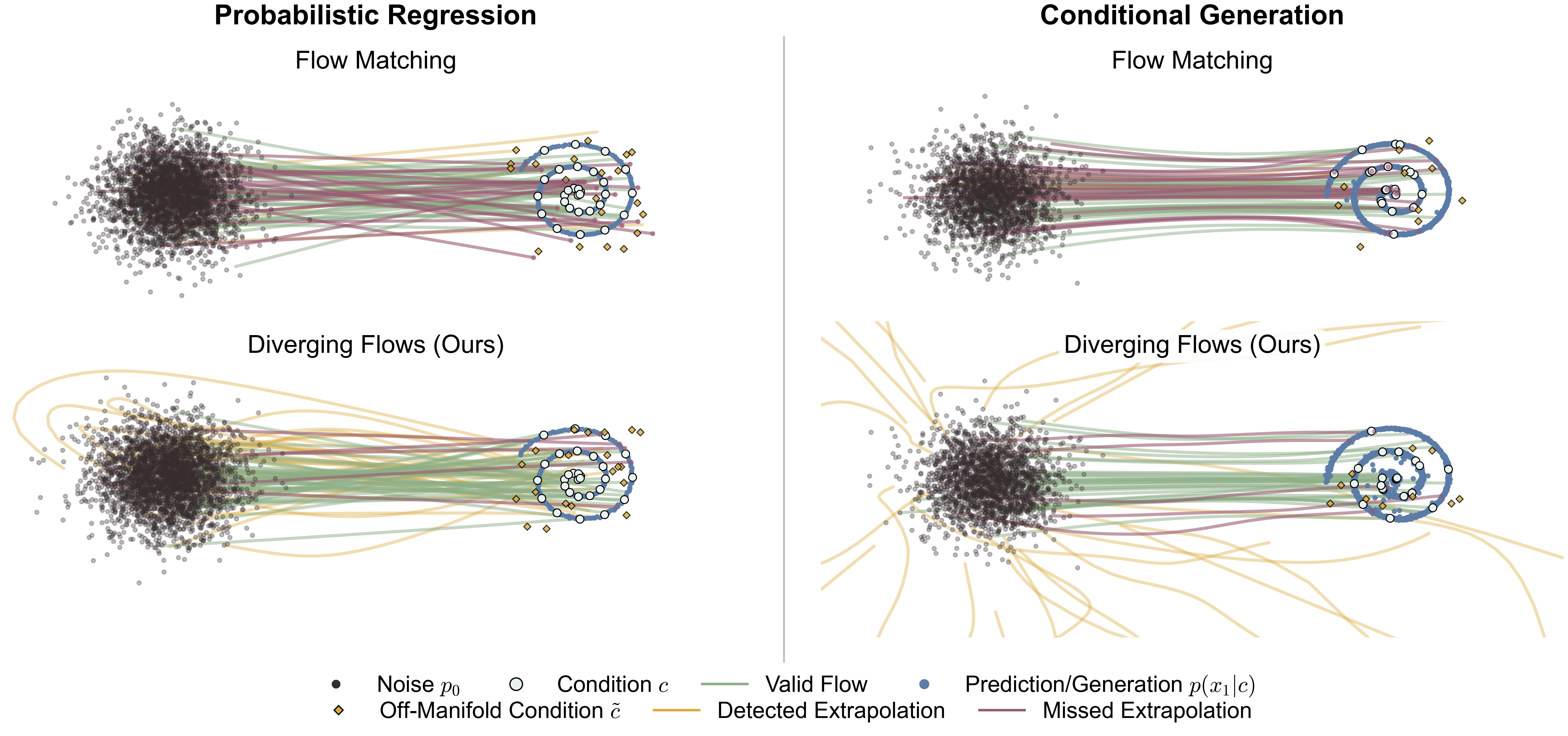}
    \captionsetup{font=small} 
    \caption{\textbf{Vector Field Dynamics on Synthetic Manifolds.} Probabilistic Regression (left, 10-step horizon)
    and Conditional Generation (right). (Top) Standard FM forces smooth convergence for all inputs, causing silent
    hallucinations. (Bottom) {\algo} enforces a conservative field: valid inputs (green) follow optimal geodesics, while
    off-manifold queries (orange) are forced into detectable divergent flows.}\label{fig:2d_spiral}
\end{figure}

To counteract the implicit smoothness bias described in Section~\ref{sec:extrapolation_problem}, we must enforce inefficient transport for off-manifold inputs. We achieve this by projecting contrastive boundaries directly into the dynamic space of vector fields. 

We treat the optimal transport geodesic $u_t^{OT}$ as a \textit{geometric anchor}. For a valid data-condition pair $(x_1, c)$, standard Flow Matching drives the predicted velocity $v^\theta_t(x_t, c)$ toward this anchor. However, for an off-manifold condition $\tilde{c} \in \mathcal{M}^c$, we apply a targeted boundary objective to induce a controlled violation of the Lipschitz continuity.

We decompose this geometric deviation into two orthogonal components. First, we define a kinetic energy margin
($\mathcal{L}_{\algotext{repel}}$) that bounds the magnitude of the off-manifold velocity:
\begin{equation}
    \mathcal{L}_{\algotext{repel}} = \Phi\Big( ||u_t^{OT} - v^\theta_t(x_t, c)||_2 - ||u_t^{OT} - v^\theta_t(x_t, \tilde{c})||_2 + m_r \Big)
    \label{eq:l_repel}
\end{equation}
This term explicitly forces the off-manifold flow to violently accelerate or decelerate relative to the optimal transport speed. By breaking the constant-velocity property, we ensure the transport becomes energetically inefficient.

Second, we define a geodesic alignment margin ($\mathcal{L}_{\algotext{curve}}$) to penalize the directional similarity of the flows:
\begin{equation}
    \mathcal{L}_{\algotext{curve}} = \Phi\Big( d_{\cos}(u_t^{OT}, v^\theta_t(x_t, c)) - d_{\cos}(u_t^{OT}, v^\theta_t(x_t, \tilde{c})) + m_c \Big)
    \label{eq:l_curve}
\end{equation}
where $d_{\cos}(a,b) = 1 - \frac{a \cdot b}{||a|| ||b||}$ measures the angular divergence. This penalty actively steers
the vector field orthogonally away from the geodesic, ensuring that even if the network attempts to maintain the correct speed, it is forced to travel in the wrong direction.

The two hyperparameters $m_r$ and $m_c$ establish energetic and directional constraints around the
valid data manifold. The function $\Phi: \mathbb{R} \to \mathbb{R}^+$ is a monotonically increasing penalty, which we
instantiate as a hinge projection $\Phi(z) = \max(0, z)$. This hard-margin formulation is essential for optimization
stability; once the off-manifold velocity deviates sufficiently to satisfy the margin, the contrastive gradient
vanishes. This bounded penalty promotes geometric divergence without triggering unbounded vector field explosion,
thereby preserving the structural integrity of on-manifold flows. The final {\algo} ({\algosmall}) objective integrates the baseline regression with these dynamic regularizers:
\begin{equation}
    \mathcal{L}_{\algotext{\algosmall}} = \mathbb{E}_{t, x_t} \left[ ||v^\theta_t(x_t, c) - u_t^{OT}||_2^2 \right] + \lambda
    \mathcal{L}_{\algotext{repel}} + \beta \mathcal{L}_{\algotext{curve}}
    \label{eq:df_loss}
\end{equation}

Crucially, the transport energy excess induced by Eq.~\ref{eq:df_loss} cannot arbitrarily dissipate; it must manifest as trajectory curvature. We formalize this geometric bound as follows:

\begin{theorem}[Geometric Divergence Guarantee]
\label{thm:realized_energy_curvature}
Let $\psi: [0,1] \times \mathbb{R}^d \rightarrow \mathbb{R}^d$ be a flow map governed by the conditional velocity field
    $\frac{d\psi_t(x)}{dt} = v_t^\theta(\psi_t(x), \tilde{c})$, mapping an initial state $x_0 \sim p_0$ to a realized
    prediction $\hat{x}_1 = \psi_1(x_0)$ under an off-manifold condition $\tilde{c} \notin \mathcal{M}$. Let $u^{local}
    = \hat{x}_1 - x_0$ be the constant-velocity optimal transport path connecting these specific endpoints. If the
    contrastive objective bounds the transport energy excess relative to this path below by a margin $\Delta_{local} >
    0$, such that $\int_0^1 ||v_t^\theta(\psi_t(x_0), \tilde{c}) - u^{local}||_2^2 dt \ge \Delta_{local}$, then the
    trajectory's total curvature (integrated acceleration) $\kappa = \int_0^1 ||\frac{d}{dt}v_t^\theta(\psi_t(x_0),
    \tilde{c})||_2^2 dt$ is lower-bounded by $\kappa \ge \pi^2 \Delta_{local}$ (Proof in Appendix
    \ref{proof:divergence_guarantee}).
\end{theorem}

Theorem~\ref{thm:realized_energy_curvature} formalizes the core mechanism of our framework. By locally
elevating the Lipschitz constant at the manifold boundary via our new losses, we explicitly induce an excess in
transport energy. As demonstrated above, this energy excess directly bounds the integrated acceleration, structurally
enforcing high curvature along the integration path. Consequently, this ensures a distinct structural divergence:
while valid conditions exhibit only the baseline curvature inherent to the marginal flow, off-manifold extrapolations
are mathematically forced into highly inefficient, turbulent trajectories, as illustrated in Figure~\ref{fig:concept_fig}. The training algorithm is presented in
Appendix~\ref{appendix:algorithms}.

\subsection{Manifold Boundary via Negative Mining}
\label{sec:negative_mining}

The effectiveness of the contrastive margins depends on constructing hard negatives that lie near, but outside, the support of valid conditions. These boundary-adjacent perturbations are critical for inducing a localized separation in the transport energy landscape around $\mathcal{M}$.

Our framework is agnostic to the specific negative mining strategy. When domain knowledge is available, one may leverage semantic perturbations, feature-space Mixup, or Virtual Outlier Synthesis (VOS) \citep{du2022vos}. In the absence of such priors, we employ adversarial Projected Gradient Descent (PGD) \citep{madry2017towards} as a domain-agnostic mechanism for generating informative negatives.

Rather than randomly sampling the ambient space, PGD explicitly mines hard negatives strictly at the manifold boundary. It achieves this by seeking perturbations within an $\epsilon$-ball that locally maximize the baseline transport energy, such that $\tilde{c}^* \approx \arg\max_{\|\tilde{c} - c\|_\infty \le \epsilon} \|v_t^\theta(x_t, \tilde{c}) - u_t^{OT}\|_2^2$. This identifies directions of maximal sensitivity. We implement this via a $K$-step iterative projected ascent under the $L_\infty$ norm, with step size $\eta = \epsilon / K$, clamping values to the valid data domain:
\begin{equation}
    \tilde{c}_{k+1} = \Pi_{\mathcal{B}_\epsilon^\infty(c)}\left( \tilde{c}_k + \eta \cdot \text{sgn}\big(
    \nabla_{\tilde{c}} \mathcal{L}_\algotext{FM}(\tilde{c}_k)\big)\right).
\end{equation}

Crucially, PGD relies on the dynamics of the $\mathcal{L}_{\algotext{\algosmall}}$ objective. If adversarial mining were applied solely
against the standard $\mathcal{L}_{\algotext{FM}}$, the resulting negatives would be weak, since FM training flattens
the energy landscape, neutralizing the local gradients. However, by integrating PGD iteratively during contrastive training, our objective
actively steepens these boundary gradients, guiding PGD to discover progressively harder adversarial
examples. We illustrate this in Appendix~\ref{sec:pgd_evolution}.

\subsection{Detecting Extrapolations via Trajectory Divergence}
\label{sec:detection}

During training, OT-FM defines linear conditional paths. However, because these independent trajectories naturally cross
in state space, the network instead learns the marginal vector field --- pulling the state toward the expected target $\mathbb{E}[x_1 \mid x_t]$. At inference, this expectation acts as a moving target. Early in the generative process ($t \approx 0$), the posterior is broad, and the velocity points toward a heavily averaged, global mean. As integration progresses, this posterior sharply concentrates toward a single materialized sample. Chasing this continuously shifting target dynamically realigns the velocity vector, which inherently curves the generated trajectories despite the linear training paths.

Although valid inference flows are not perfectly straight, we seek to quantify the trajectory's macroscopic deviation
induced by the contrastive objective. Thus, we define a constant-velocity geometric anchor to serve as a
fixed reference metric.  For a generated
sequence of states $\{\hat{x}_{t_i}\}_{i=1}^N$ solved over $N$ discrete time steps from $x_0$ to the final prediction
$\hat{x}_1$, this linear interpolant is defined as: $x^{\text{ref}}_{t_i} = (1-t_i)x_0 + t_i \hat{x}_1$. 

We measure this deviation using the \textit{Divergence from Optimal Trajectory} (DOT) score, defined as the
spatiotemporal average of the $L_1$ distance between the realized flow and this reference anchor:
\begin{equation}\label{eq:dot_score}
S_{DOT}(\hat{x}_{0:1}) = \sum_{i=1}^{N} \left( \frac{1}{D} \sum_{d=1}^{D} | \hat{x}^{(d)}_{t_i} - x^{\text{ref}^{(d)}}_{t_i} | \right)
\end{equation}
where $D$ is the spatial dimensionality. While on-manifold flows naturally exhibit minor curvature ($S_{DOT} > 0$),
extrapolation flows exhibit massive, mathematically enforced geometric divergence, resulting in significantly larger DOT scores ($S_\text{DOT}(c) \ll
S_\text{DOT}(\tilde{c})$). 

Because valid, on-manifold inputs possess a natural distribution of DOT scores rather than a strict zero value,
establishing a heuristic detection threshold is suboptimal. To ensure safety-critical reliability, we employ Split
Conformal Prediction \citep{angelopoulos2024theoretical} on a calibration set of valid conditions. This framework
calculates a statistically rigorous decision boundary, providing a distribution-free, finite-sample guarantee that the
false rejection rate for valid inputs remains bounded by a user-specified error rate $\alpha$ during
deployment. Crucially, because the DOT metric relies solely on lightweight $L_1$ distance computations against a
geometric anchor, it introduces near-zero computational overhead ($\sim 0.8\%$) to the standard generative inference
pass. Further details on the conformal calibration and inference algorithms are provided in
Appendices~\ref{appendix:split_cp} and \ref{appendix:algorithms}.

\section{Experiments}\label{experiments}

We evaluate {\algo} across three distinct regimes that showcase the geometric behavior and deployment viability: (i)
geometric verification on synthetic manifolds, (ii) weather temperature forecasting, and (iii) cross-domain style
transfer. Since no existing FM or Diffusion models possess built-in extrapolation awareness, we benchmark primarily
against standard FM. To evaluate native extrapolation detection, we exclude methods requiring auxiliary
networks or post-hoc pipelines, instead evaluating the baseline via two distinct scoring mechanisms. First, we employ
our proposed DOT score on the standard model; this verifies that detection capabilities stem from the
contrastive reshaping of the vector field rather than the metric alone. Second, we compare against the FM model's
likelihood, representing the conventional statistical approach. Following \citet{lipman2024flow}, we estimate this using
the instantaneous change of variables formula with Hutchinson’s trace estimator. Furthermore, we benchmark against an FM
Ensemble with a depth of three, utilizing the ensemble's variance as the extrapolation score. We limit this comparison
to regression tasks; in conditional generation tasks, a single valid condition naturally maps to a diverse distribution of outputs, rendering ensemble variance an unreliable indicator of off-manifold hallucination. For completeness, we provide a pure detection benchmark comparing {\algo} against the DiffPath detector in Appendix~\ref{appendix:diffpath}.


We quantify \textit{Safety} via two metrics: (1) the AUROC \citep{fawcett2006introduction}, measuring global
separability; and (2) the FPR governed by split conformal prediction constraints. We calibrate a decision region
$C_{\alpha}$ to guarantee a marginal coverage of $1-\alpha$ for on-manifold inputs. The FPR reports the
percentage of off-manifold anomalies $\tilde{c}$ that incorrectly satisfy this condition, defined formally as
$\text{FPR} = \mathbb{P}(S(\tilde{c}) \in C_{\alpha} \mid \tilde{c} \notin \mathcal{M})$. For
\textit{Fidelity}, we ensure our contrastive objective does not degrade predictive performance using task-specific
metrics. In regression, we monitor physical consistency via Mean Squared Error (MSE), Peak Signal-to-Noise Ratio (\text{PSNR})~\citep{hore2010image}, and Structural Similarity (\text{SSIM})~\citep{wang2003multiscale, wang2004image}. In generative tasks, we assess distribution quality via the Fréchet Inception Distance~\citep{heusel2017gans}, defined as $\text{FID} = \|\mu_r - \mu_g\|^2 + \text{Tr}(\Sigma_r + \Sigma_g - 2(\Sigma_r \Sigma_g)^{1/2})$, and Perceptual Similarity (LPIPS)~\citep{zhang2018unreasonable}. Detailed hyperparameters and architecture specifications are provided in Appendix \ref{appendix:exp_details}.

\subsection{Synthetic Manifold Experiments}\label{sec:spiral_sec}

Let $\mathcal{M}^2$ be a spiral manifold parameterized by $\theta \sim
\mathcal{U}[0, 5\pi]$. A valid sample $x(\theta)$ is defined as: $x(\theta) = r(\theta)(\cos \theta \; \sin
\theta)^\intercal + \xi$, where $r(\theta) = \frac{\theta}{5\pi}$ is the normalized radius and $\xi \sim \mathcal{N}(0,
\sigma^2 I)$ represents intrinsic noise ($\sigma_\text{gen} = 5\cdot10^{-3}$ and $\sigma_\text{reg} = 2\cdot10^{-3}$).
The invalid set is defined as the ambient space separated from the manifold support:
$\mathcal{C}^c = \{c \in [-1, 1]^2 \mid \min_{c' \in \mathcal{M}} \|c - c'\|_2 > \epsilon\}$.

We evaluate two conditional tasks: (i) \textbf{Probabilistic Regression:} Given a current state $c = x(\theta_k)$, the
model predicts the future trajectory $y = [x(\theta_{k+1}), \dots, x(\theta_{k+10})]$. This tests the preservation of
local dynamics alongside the ability to reject invalid inputs. (ii) \textbf{Conditional Generation:} The model learns
$p_\theta(x|c) \approx p_1(x)$ for any valid condition $c \in \mathcal{M}$, subject to $x \neq c$. Here, $c$ acts
as a validity token; off-manifold conditions must trigger vector field divergence rather than sampling valid data. For
the contrastive training, we employ PGD to mine negative samples. The
results of the experiments are illustrated in Figure~\ref{fig:2d_spiral}, where we show how the flow ODE of FM and {\algosmall}
behaves when an off-manifold input is encountered.

\begin{table}[h]
\centering
    \captionsetup{font=small} 
\caption{\textbf{Results on Synthetic Manifolds}. FM fails to detect off-manifold inputs effectively, while
    {\algo} achieves consistently high separation with low FPR after the conformal calibration. We report the mean and standard deviation across three random seeds.}
\label{tab:2d_results}
\scriptsize
\renewcommand{\arraystretch}{0.85} 
\setlength{\tabcolsep}{4pt} 
\begin{tabular*}{\textwidth}{@{\extracolsep{\fill}}llccc}
\toprule
    & \algotext{Algorithm} & \algotext{AUROC} $\uparrow$ & \algotext{FPR(\%)} $\downarrow$ & \algotext{MSE }($10^{-5}$) $\downarrow$ \\
\midrule

    \multirow{3}{*}{\rotatebox[origin=c]{90}{\algotext{Gen.}}} 
    & \algotext{FM-Likelihood} & \mest{0.493}{0.004} & \mest{97.0}{0.4} & -- \\
    & \algotext{FM-DOT}        & \mest{0.495}{0.006} & \mest{98.0}{0.2} & -- \\
    & \algoours{(Ours)} & \textbf{\mest{0.981}{0.004}} & \textbf{\mest{4.6}{1.0}} & -- \\

\midrule

\multirow{4}{*}{\rotatebox[origin=c]{90}{\algotext{Regr.}}}
& \algotext{FM-Likelihood} & \mest{0.566}{0.02}          & \mest{98.9}{0.2}          & \mest{8.43}{0.09} \\
& \algotext{FM-DOT}        & \mest{0.594}{0.02}          & \mest{94.5}{0.7}          & \mest{8.43}{0.09} \\
& \algotext{FM Ensemble}   & \mest{0.688}{0.01}          & \mest{88.8}{2.0}          & \textbf{\mest{8.40}{0.08}} \\
& \algoours{(Ours)}        & \textbf{\mest{0.998}{0.00}} & \textbf{\mest{1.00}{0.0}} & \mest{8.48}{0.10} \\

\bottomrule
\end{tabular*}
\end{table}

Results on synthetic manifolds (Table ~\ref{tab:2d_results}) empirically confirm the smoothness hazard inherent to Flow Matching. The baseline
yields AUROC scores near $0.50$ for both tasks, effectively reducing off-manifold detection to random guessing.
While the ensemble performs slightly better, it remains an unreliable detector. This
proves that without explicit constraints, the model extrapolates smooth paths even for invalid inputs. In
contrast, {\algo} induces a sharp geometric transition, achieving consistently high separation while
maintaining predictive fidelity comparable to the baseline. This confirms that
safety can be embedded into the vector field without compromising in-distribution dynamics. Visualizations of the
learned detection landscapes across the full ambient space are provided in Appendix~\ref{appendix:2d_landscape}. In
Appendix~\ref{appendix:sensitivity_2d}, we perform an ablation study to isolate the contribution of each loss component
and validate the necessity of a targeted negative mining strategy. Additionally, we provide a comprehensive sensitivity analysis to demonstrate the robustness of the model's hyperparameters across these tasks.

\subsection{Weather Temperature Forecasting}\label{sec:weather_for}

We evaluate the deployment viability of {\algo} on the ERA5 dataset~\citep{era5paper}, by forecasting surface
temperature heatmaps ($64 \times 64$) six hours into the future. The model $p_\theta(x_{t+6h}|x_t)$ is parameterized as
a U-Net, where the current atmospheric state $x_t$ is injected via cross-attention layers to condition the generation of
the future state. For this task, we utilize PGD for mining negative samples.

\begin{figure}[h]
    \centering
    \includegraphics[width=0.60\textwidth]{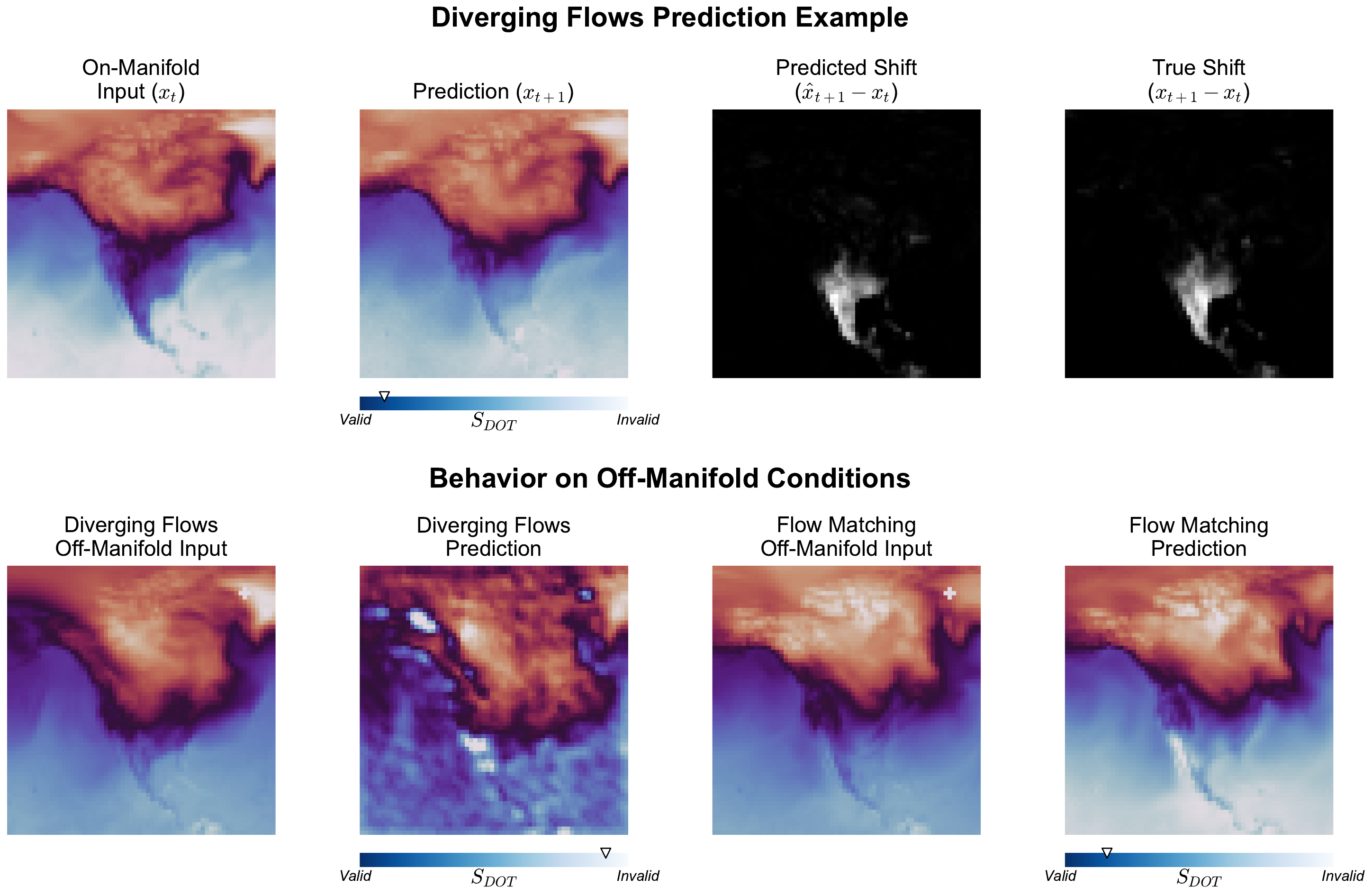}
    \captionsetup{font=small} 
    \caption{\textbf{High-Fidelity Forecasts with {\algo}.} (Top) {\algo} preserves exact physical dynamics for valid
    inputs. (Bottom) Off-manifold thermal anomalies trigger mathematically enforced trajectory divergence, breaking the
    flow, while standard FM hallucinates silently.}\label{fig:weather_fig}
\end{figure}

This domain demands strict adherence to physical laws; a reliable model must reject inputs that violate thermodynamic principles. 
To test the physical consistency, we adopt the artificial hotspot setup from \citet{chan2024estimating}. We perturb
valid inputs by injecting thermal anomalies into historically cold regions. These perturbations are semantically
plausible as images but violate local thermodynamic laws. A conservative forecaster must identify these inputs as
physically infeasible rather than hallucinating. We compare our method against standard FM, an FM Ensemble, and HyperDM
(the diffusion baseline from the original benchmark). Note that HyperDM lacks native off-manifold detection and is
included as a reference for predictive fidelity.

\begin{table}[h]
\centering
    \captionsetup{font=small} 
    \caption{\textbf{Weather Temperature Forecasting Performance (ERA5)}. {\algosmall} achieves robust awareness of
    physical anomalies while maintaining competitive predictive fidelity on valid forecasts. We report the mean and standard deviation across three random seeds.}
\label{tab:weather_forec}
\scriptsize
\renewcommand{\arraystretch}{0.85} 
\setlength{\tabcolsep}{4pt} 

\begin{tabular*}{\textwidth}{@{\extracolsep{\fill}}lcccc}
\toprule
\algotext{Algorithm} & \algotext{MSE} $\downarrow$ & \algotext{SSIM} $\uparrow$ & \algotext{PSNR} $\uparrow$ & \algotext{AUROC} $\uparrow$ \\
\midrule

\algotext{HyperDM}               & \ttp{4}{-3}                      & 0.954                          & 33.15                         & --- \\
\algotext{FM-Likelihood}         & \ttpstd{8.1}{1.1}{-4}            & \mest{0.970}{0.001}            & \mest{36.97}{0.61}            & \mest{0.512}{0.02}  \\
\algotext{FM-DOT}                & \ttpstd{8.1}{1.1}{-4}            & \mest{0.970}{0.001}            & \mest{36.97}{0.61}            & \mest{0.599}{0.01}  \\
\algotext{FM Ensemble}           & \textbf{\ttpstd{6.7}{0.00}{-4}}  & \textbf{\mest{0.975}{0.001}}   & \textbf{\mest{37.73}{0.65}}   & \mest{0.863}{0.08} \\
\algoours{(Ours)}                & \ttpstd{8.22}{0.3}{-4}           & \mest{0.967}{0.002}            & \mest{36.36}{0.21}            & \textbf{\mest{0.991}{0.01}} \\

\bottomrule
\end{tabular*}
\end{table}

The results on ERA5 (Table~\ref{tab:weather_forec}) show that standard models fail to capture physical constraints. The
FM baseline
yields a near-random AUROC because it ignores the violation, treating the hotspot as valid high-frequency noise, as
illustrated in Figure~\ref{fig:weather_fig}. The ensemble exhibits high performance across all predictive metrics and
sufficient detection capability; however, this imposes massive training and inference overhead.
Compared to the HyperDM baseline, flow-based methods perform slightly better, in terms of predictive fidelity.

In contrast, {\algo} achieves near-perfect detection. By altering the transport for invalid inputs, the network
effectively learns to verify physical validity. Critically, we observe that {\algo} contributes to nearly-zero
degradation in the predictive performance compared to the standard FM. This confirms that geometric
regularization effectively filters unphysical dynamics without hindering the learning of valid temperature patterns.
In Appendix~\ref{app:far_off_manifold}, we further demonstrate
that the learned boundary effectively generalizes to detect extreme, far off-manifold semantic shifts. 
Lastly, conformal prediction ($\alpha=0.05$) yields a 1.6\% FPR, providing a rigorous statistical guarantee of 95\%
coverage while identifying thermodynamic extrapolations. To evaluate the robustness of {\algo} in higher dimensions, we perform a
sensitivity analysis on this task in Appendix~\ref{appendix:sensitivity_era5}.

\subsection{Cross-Domain Style Transfer}\label{sec:style_trans}

Lastly, we investigate the model's ability to handle structurally disjoint conditioning and training manifolds by mapping grayscale
MNIST digits ($c$)~\citep{LeCun2005TheMD} to RGB Street House View Numbers (SVHN) ($x$)~\citep{netzer2011reading}. We
introduce Fashion-MNIST (FMNIST)~\citep{xiao2017fashion} and Kuzushiji-MNIST (KMNIST)~\citep{clanuwat2018deep} to probe
whether {\algo} implicitly captures the semantic support of valid conditions. FMNIST represents a clear
\textit{structural shift} (clothing vs. digits), whereas KMNIST poses a harder \textit{semantic shift}, sharing
significant low-level statistics (e.g., stroke density) with MNIST. To evaluate the flexibility in negative mining for
this visual domain, we compare both PGD and semantic transformations. We describe the transformations we used along with
an ablation study on additional mining techniques in Appendix~\ref{appendix:mining_abl}.

\begin{table}[h]
\centering
    \captionsetup{font=small} 
    \caption{\textbf{Cross-Domain Style Transfer Results}. {\algosmall} successfully learns a semantic boundary,
    rejecting invalid queries without degrading generative quality (FID) compared to the baseline.}
\label{tab:style_results}
\scriptsize 
\renewcommand{\arraystretch}{0.85} 
\setlength{\tabcolsep}{4pt} 

\begin{tabular*}{\textwidth}{@{\extracolsep{\fill}}lccc}
\toprule
\algotext{Algorithm} & \algotext{FID-50K} $\downarrow$ & \algotext{LPIPS} & \algotext{AUROC (FMNIST/KMNIST)} $\uparrow$ \\
\midrule
\algotext{FM-Likelihood} & \textbf{4.102} & 0.2172 & 0.529 / 0.516 \\
\algotext{FM-DOT}        & \textbf{4.102} & 0.2172 & 0.527 / 0.513 \\
\algoours{PGD} & 4.104 & 0.2202 & \textbf{0.955} / \textbf{0.860} \\
\algoours{Transformations} & 4.109 & 0.2213 & \textbf{0.987} / \textbf{0.892} \\
\bottomrule
\end{tabular*}
\end{table}

\begin{wrapfigure}{r}{0.5\textwidth}
    \centering
    \includegraphics[width=0.9\linewidth]{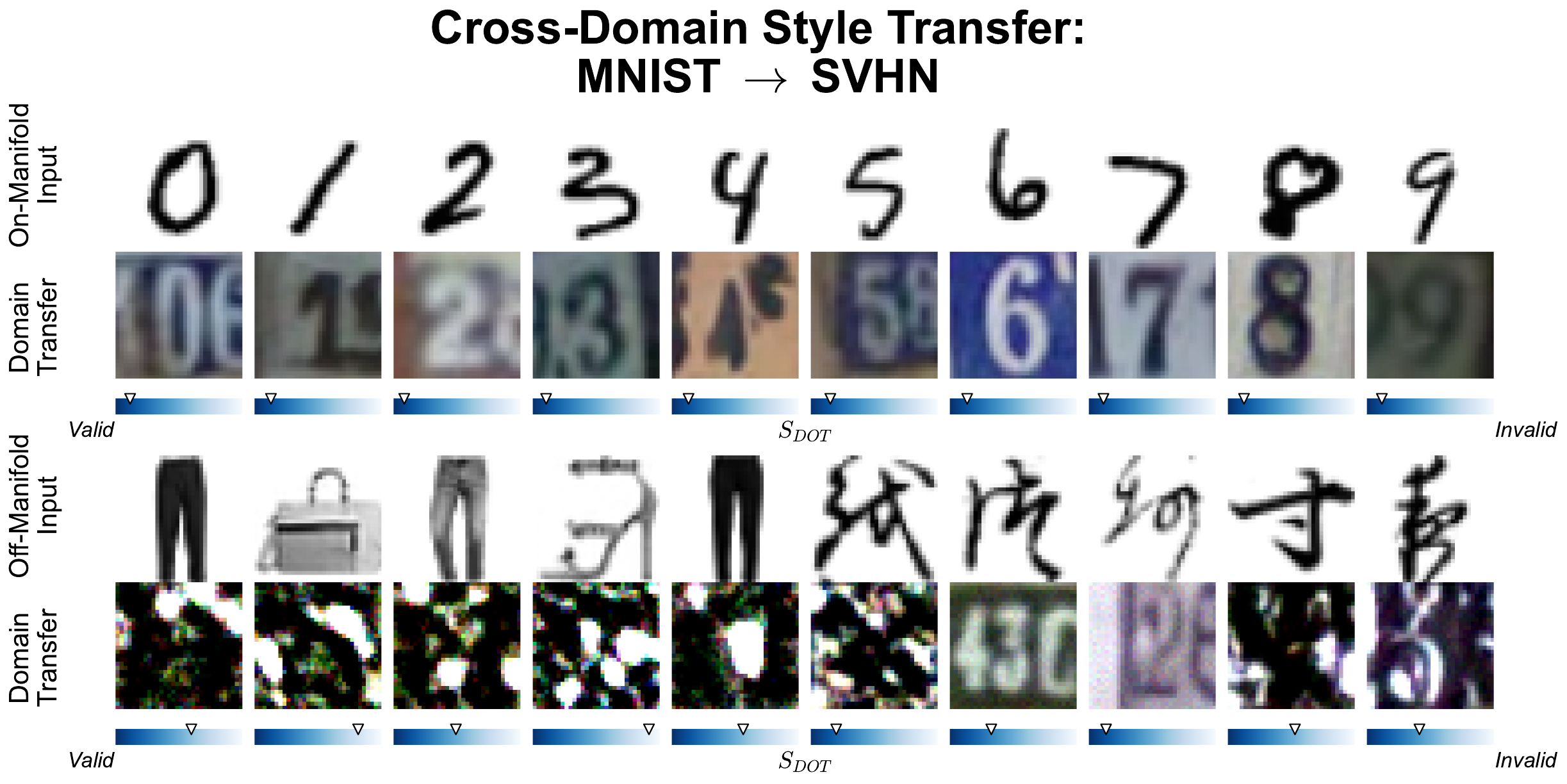}
    \captionsetup{font=small} 
    \caption{\textbf{Generative Quality in Cross-Domain Transfer.} Samples mapping MNIST $\to$ SVHN. {\algo}
    preserves the semantic identity of the digit.}
    \label{fig:mnist_to_svhn_gens}
\end{wrapfigure}

The style transfer results (Table~\ref{tab:style_results}) highlight that FM ignores semantic boundaries, since it
extrapolates confident paths for KMNIST (near-random AUROC). In contrast, {\algo} imposes a robust boundary via PGD or Transformations,
learning to reject non-digit concepts. Crucially, this awareness comes with negligible cost to
generation quality; as shown in Figure~\ref{fig:mnist_to_svhn_gens}, the model retains expressivity for valid
inputs, confirming that geometric regularization targets invalid semantics without constraining the valid generative
distribution.

\section{Conclusion and Limitations}\label{conclusion}

{\algo} allows a single Flow Matching model to possess native extrapolation awareness while performing high-fidelity
prediction tasks. Our core contribution is the formulation of a
geometric phase transition: by actively enforcing trajectory divergence for off-manifold inputs, we transform the vector
field from a generator into a self-monitoring mechanism. This shifts extrapolation detection from a post-hoc statistical
estimation problem to a geometric constraint, ensuring that safety is an inherent property of the dynamics
rather than an external add-on. Crucially, this design of {\algo} introduces negligible inference overhead,
establishing it as a robust solution for real-time critical applications.

A limitation of our approach is the training overhead required to create hard negatives in the case of an iterative
process like PGD. Additionally, our current formulation relies on Euclidean Optimal Transport; but the contrastive principles
developed here can easily be extended to Riemannian manifolds, which would make it effective to support geometric deep
learning tasks~\citep{chen2024flow}. Finally, as a next step, we plan to deploy {\algo} in real-world robotic tasks, where detecting
extrapolations is essential to avert potentially harmful outcomes.

\section*{Acknowledgments}
This work was supported by a DGA-Inria contract (ATOR
Project), the EU Horizon project euROBIN (GA n.101070596), and
the France 2030 program through the PEPR O2R projects
AS3 and PI3 (ANR-22-EXOD-007, ANR-22-EXOD).

\medskip

{
\small

\bibliographystyle{unsrtnat}
\bibliography{references}
}


\newpage
\appendix

\section{Broader Impact}
\label{appendix:impact}
The goal of this work is to advance the reliability of generative models in scientific and industrial tasks. By enabling
models to detect and reject extrapolations, our method mitigates the risks of silent hallucinations in
safety-critical applications such as weather forecasting, autonomous systems, and scientific discovery. We believe this
work contributes positively to the responsible deployment of machine learning in high-stakes domains by transforming
invisible failure modes into detectable signals. We do not foresee specific negative societal consequences, as the
primary focus of this work is to advance machine learning for reliability and transparency.

\section{Remarks and Proofs}
\label{appendix:proofs}

\subsection{Proof of Proposition \ref{prop:lipschitz_energy}}

\begin{proof}
\label{proof:lipschitz_energy}
To determine the Lipschitz constant $L$, we must bound the absolute difference $|\mathcal{E}(\tilde{c}) - \mathcal{E}(c)|$ for a perturbed off-manifold condition $\tilde{c}$ relative to a valid on-manifold condition $c$. By the triangle inequality for expectations:
\begin{equation}
    |\mathcal{E}(\tilde{c}) - \mathcal{E}(c)| \le \mathbb{E}_{t, x_t} \left[ \left| ||v_t^\theta(x_t, \tilde{c}) - u^{OT}||_2^2 - ||v_t^\theta(x_t, c) - u^{OT}||_2^2 \right| \right]
\end{equation}
For any fixed $t$ and $x_t$, we can rewrite the difference of squared $L_2$ norms using the properties of inner products. Applying the Cauchy-Schwarz inequality followed by the triangle inequality yields:
\begin{equation}
\begin{aligned}
    & \left| ||v_t^\theta(x_t, \tilde{c}) - u^{OT}||_2^2 - ||v_t^\theta(x_t, c) - u^{OT}||_2^2 \right| \\
    &= \left| \langle v_t^\theta(x_t, \tilde{c}) - v_t^\theta(x_t, c), v_t^\theta(x_t, \tilde{c}) + v_t^\theta(x_t, c) - 2u^{OT} \rangle \right| \\
    &\le ||v_t^\theta(x_t, \tilde{c}) - v_t^\theta(x_t, c)||_2 \cdot ||v_t^\theta(x_t, \tilde{c}) + v_t^\theta(x_t, c) - 2u^{OT}||_2 \\
    &\le ||v_t^\theta(x_t, \tilde{c}) - v_t^\theta(x_t, c)||_2 \left( ||v_t^\theta(x_t, \tilde{c})||_2 + ||v_t^\theta(x_t, c)||_2 + 2||u^{OT}||_2 \right)
\end{aligned}
\end{equation}
By hypothesis, the vector field is locally Lipschitz continuous with constant $K$, such that $||v_t^\theta(x_t, \tilde{c}) - v_t^\theta(x_t, c)||_2 \le K||\tilde{c} - c||_2$. Furthermore, the vector field is uniformly bounded, such that $||v_t^\theta(\cdot)||_2 \le M$. Substituting these upper bounds and factoring the deterministic constants out of the expectation gives:
\begin{equation}
\begin{aligned}
    |\mathcal{E}(\tilde{c}) - \mathcal{E}(c)| &\le \mathbb{E}_{t, x_t} \left[ K||\tilde{c} - c||_2 (2M + 2||u^{OT}||_2) \right] \\
    &= 2K(M + ||u^{OT}||_2) ||\tilde{c} - c||_2
\end{aligned}
\end{equation}
Thus, the Transport Energy Excess $\mathcal{E}(c)$ is Lipschitz continuous with constant $L \le 2K(M + ||u^{OT}||_2)$.
\end{proof}

\begin{remark}[The Mechanics of Silent Extrapolation via Spectral Bias]
It is well-established that the implicit regularization of deep neural networks—specifically their spectral bias \citep{rahaman2019spectral}—favors the learning of smooth, low-frequency functions. Standard Flow Matching minimizes the baseline energy uniformly over the training distribution, which concentrates heavily on the valid condition manifold $\mathcal{M}$. Due to the spectral bias, the learned vector field preferentially fits these low-frequency components, implying that the Lipschitz constant remains small not only globally, but specifically in the local neighborhood of $\mathcal{M}$. Formally, for any $c \in \mathcal{M}$ and off-manifold query $\tilde{c} \in \mathcal{B}_\delta(c) \cap \mathcal{M}^c$, the local Lipschitz constant satisfies $K_{\text{loc}} \le K + \mathcal{O}(\delta)$. 

Proposition \ref{prop:lipschitz_energy} demonstrates that this bounded $K_{\text{loc}}$ mathematically forces the energy landscape to remain flat. Because standard Flow Matching explicitly minimizes $\mathcal{E}(c)$ for valid conditions, this tight continuity bound guarantees that for any proximate off-manifold condition $\tilde{c}$, the transport energy $\mathcal{E}(\tilde{c})$ cannot increase sufficiently to signal an extrapolation. This dynamic formally establishes the structural root of the silent extrapolation hazard: without active geometric regularization, the inherent smoothness of the neural vector field fundamentally masks invalid inputs.
\end{remark}

\subsection{Proof of Theorem \ref{thm:realized_energy_curvature}}
\label{proof:divergence_guarantee}
\begin{proof}
The generative flow must continuously connect the initial state $x_0$ to its final predicted state $\hat{x}_1$. We can decompose the learned conditional velocity field, evaluated along this generated trajectory, into the constant-velocity baseline of the realized path and a time-dependent geometric deviation:
\begin{equation}
    v_t^\theta(\psi_t(x_0), \tilde{c}) = u^{\text{local}} + \tilde{\epsilon}_t
\end{equation}
Because the integral of the total velocity over time must exactly equal the linear spatial displacement between the endpoints, we have:
\begin{equation}
    \int_0^1 (u^{\text{local}} + \tilde{\epsilon}_t) dt = \hat{x}_1 - x_0 = u^{\text{local}} \implies \int_0^1 \tilde{\epsilon}_t dt = 0
\end{equation}
Thus, the geometric deviation $\tilde{\epsilon}_t$ has a temporal mean of zero.

\textit{Regularity of $\tilde{\epsilon}_t$.} To rigorously apply the Poincar\'{e}--Wirtinger inequality, we must first establish that the error trajectory $\tilde{\epsilon}_t \in H^1([0,1]; \mathbb{R}^d)$. By construction, our parameterized vector field $v_t^\theta(x, c)$ utilizes smooth activation functions (e.g., SiLU, GELU) and differentiable time embeddings, rendering it continuously differentiable with respect to both time and state. Consequently, the flow $v_t^\theta(\psi_t(x_0), \tilde{c})$ is continuously differentiable in $t$ along the generated trajectory $\psi_t(x_0)$. Since the baseline transport vector $u^{\text{local}}$ is constant with respect to time, the deviation $\tilde{\epsilon}_t$ inherently preserves this continuous differentiability. Furthermore, because the vector field is uniformly bounded and the trajectory resides entirely within a compact domain (as inputs are normalized to $[-1,1]^d$), the total time derivative $\dot{\tilde{\epsilon}}_t$ is square-integrable on $t \in [0,1]$. Therefore, the $H^1$ regularity requirement is inherently satisfied.

By hypothesis, the transport energy excess relative to the linear path is bounded below:
\begin{equation}
    \int_0^1 \|\tilde{\epsilon}_t\|_2^2 dt \ge \Delta_{\text{local}}
\end{equation}

Since $u^{\text{local}}$ is a constant vector field, the acceleration of the trajectory is exactly the temporal derivative of the deviation: $\frac{d}{dt}v_t^\theta = \dot{\tilde{\epsilon}}_t$. Applying the Poincar\'{e}--Wirtinger inequality to the zero-mean function $\tilde{\epsilon}_t$, and noting that $\pi^2$ is exactly the first non-zero eigenvalue of the Neumann Laplacian on the interval $[0,1]$, we obtain the sharp bound:
\begin{equation}
\int_0^1 \|\tilde{\epsilon}_t\|_2^2 dt \le \frac{1}{\pi^2} \int_0^1 \|\dot{\tilde{\epsilon}}_t\|_2^2 dt
\end{equation}
Substituting the transport energy boundary yields the curvature lower bound $\kappa \ge \pi^2 \Delta_{\text{local}}$. Therefore, any excess kinetic energy injected into the system that does not contribute to the final spatial displacement must manifest as path curvature.
\end{proof}

\begin{remark}[Dependence of $u^{\text{local}}$ on $\tilde{c}$]
We note that $u^{\text{local}} = \hat{x}_1 - x_0$ depends implicitly on the off-manifold condition $\tilde{c}$ through the generated endpoint $\hat{x}_1$. Consequently, the curvature lower bound $\kappa \ge \pi^2 \Delta_{\text{local}}$ is trajectory-specific rather than universal. However, this does not weaken our detection guarantees: for any realized trajectory conditioned on $\tilde{c} \in \mathcal{M}^c$, the contrastive objective structurally enforces $\Delta_{\text{local}} > 0$ relative to that trajectory's own specific linear baseline. This guarantees detectable curvature regardless of where the endpoint $\hat{x}_1$ lands in state space. The DOT score directly mirrors this property by measuring deviation relative to the realized endpoint $\hat{x}_1$, ensuring perfect alignment between our theoretical bounds and our empirical detection metric.
\end{remark}

\subsection{A remark on the margin losses}
Penalizing the vector field at an off-manifold condition $\tilde{c}$ risks inadvertently "bleeding" into the flow of the
nearby valid condition $c$ due to network continuity. This is mitigated by two mechanisms. First, our margin-based
repulsion loss $\mathcal{L}_{\algotext{repel}}$ zero-outs its gradient once the minimum required transport divergence is
met, preventing unbounded local distortion. Second, assuming the network is locally Lipschitz, the spatial bleed scales
proportionally with the distance to the negative sample. For instance, if an adversarial strategy like PGD is used to
mine negatives within an $\epsilon$-ball, the flow deviation at the valid condition is bounded by
$\mathcal{O}(\epsilon)$. By enforcing a tight search radius (Appendix~\ref{appendix:sensitivity_2d}), the valid
predictive flow remains safely isolated, a safeguard empirically verified by our fidelity metrics across all tasks.

\section{Synergy of PGD and Contrastive Objectives}
\label{sec:pgd_evolution}

We visualize the dynamic synergy between our contrastive margins and PGD to demonstrate how the model explicitly learns the manifold boundary. Figures \ref{fig:pgd_evolution_fm} and \ref{fig:pgd_evolution} track the spatial distribution of off-manifold conditions generated by PGD on the 2D Spiral manifold across training steps for both standard FM and {\algo}.

Under standard FM training (Figure \ref{fig:pgd_evolution_fm}), the network's transport energy landscape remains inherently flat due to the smoothness bias of the baseline objective. Consequently, when PGD perturbations optimize to maximize this baseline transport energy, the process lacks strong directional guidance. The resulting negative samples remain diffusely scattered around the valid manifold $\mathcal{M}$. Even after 100K training steps, PGD fails to outline a distinct boundary, yielding weak negatives that provide no meaningful safety constraints.

In contrast, under {\algo} (Figure \ref{fig:pgd_evolution}), the model actively shapes the energy landscape. At initialization (Step 0), the negatives are similarly diffuse. However, as training progresses, the structural penalties ($\mathcal{L}_\algotext{repel}$ and $\mathcal{L}_\algotext{curve}$) penalize these initial points, which systematically steepens the local gradients around the manifold support. This dynamic steepening guides PGD to discover progressively harder, highly localized negative samples. By 50K steps, these adversarial samples tightly trace the exact contour of the valid conditions, establishing the sharp geometric transition required for robust extrapolation awareness.

\begin{figure}[h]
    \centering
    \includegraphics[width=0.6\textwidth]{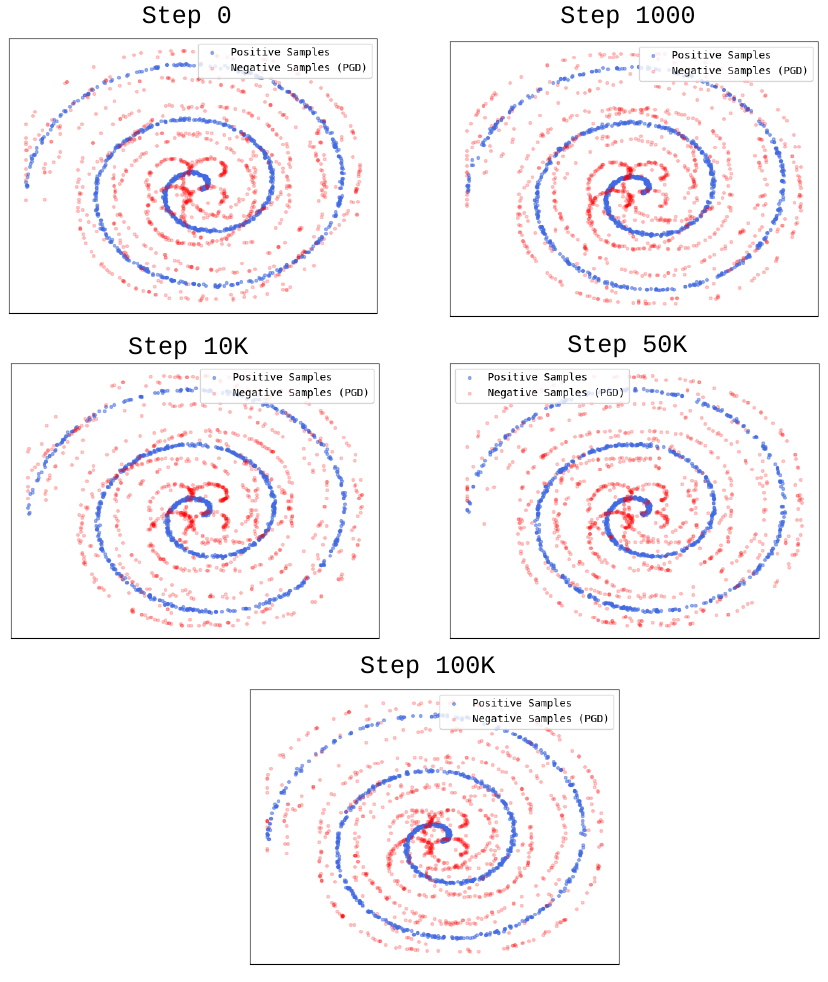}
    \caption{\textbf{Evolution of PGD-mined negative samples during standard FM training on the 2D Spiral.} Blue points represent valid on-manifold conditions ($c \in \mathcal{M}$), while red points denote adversarial off-manifold queries ($\tilde{c} \notin \mathcal{M}$). Because standard FM flattens the energy landscape, local gradients are weak; PGD fails to find a sharp boundary, leaving the negative samples diffusely scattered even at late training stages.}
    \label{fig:pgd_evolution_fm}
\end{figure}

\begin{figure}[h]
    \centering
    \includegraphics[width=0.6\textwidth]{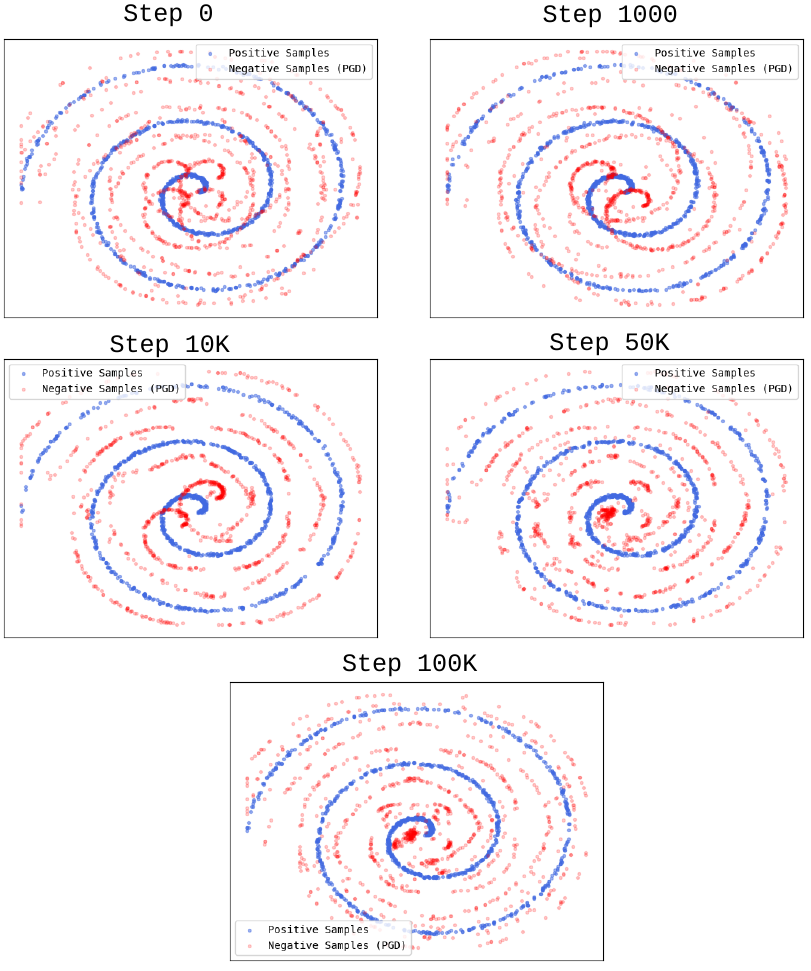}
    \caption{\textbf{Evolution of PGD-mined negative samples during {\algo} training on the 2D Spiral.} Blue points represent valid on-manifold conditions ($c \in \mathcal{M}$), while red points denote adversarial queries ($\tilde{c} \notin \mathcal{M}$). Driven by the contrastive margins, the local energy landscape steepens over time, actively guiding PGD to discover hard negatives that tightly trace the precise boundary of the valid data support.}
    \label{fig:pgd_evolution}
\end{figure}

\section{Generalization to Far Off-Manifold Conditions}
\label{app:far_off_manifold}

A standard critique of targeted negative mining strategies, such as adversarial PGD, is the
risk of manifold boundary overfitting. By explicitly optimizing the contrastive objective (Eq.~\ref{eq:df_loss}) on perturbations that lie in the local neighborhood ($\epsilon$-ball) of the valid data support $\mathcal{M}$, one might hypothesize that the model only learns a localized threshold of extrapolation awareness. If true, this would leave the deep ambient space (conditions residing far off-manifold) vulnerable to the same silent extrapolation hazards observed in unregularized Flow Matching models.

However, the geometric phase transition formalized in Section~\ref{methodology} structurally prevents this failure mode. By inducing extreme transport inefficiency at the boundary $\partial\mathcal{M}$, our contrastive objective effectively forces the model to implicitly learn the exact topological support of the valid conditioning distribution.

\paragraph{Theoretical Extension.} As established in Proposition~\ref{prop:lipschitz_energy}, the learned neural vector
field is constrained by a local Lipschitz continuity that heavily favors smoothness. Because the vector field cannot
undergo infinite-frequency spatial transitions, the massive transport energy excess injected at the boundary
structurally prevents the surrounding ambient space from collapsing back into straight, optimal-transport geodesics.
Furthermore, standard Flow Matching only minimizes transport energy on the valid support $\mathcal{M}$. Consequently,
there is no restorative geometric force (no explicitly supervised OT target) in the deep ambient space $\mathcal{M}^c$ pulling the vector field back to
efficiency. The trajectory curvature mathematically enforced at the boundary via Theorem~\ref{thm:realized_energy_curvature} must therefore propagate outward, ensuring that far off-manifold trajectories remain highly turbulent and inherently detectable via the DOT score.

\paragraph{Empirical Validation.} To empirically validate this global geometric constraint, we evaluate our {\algo} model—trained on the ERA5 Weather Temperature Forecasting dataset using local PGD mining
($\epsilon=0.1$)—against severe far off-manifold conditioning inputs. We subject the physical weather model to
conditions drawn from entirely disjoint semantic manifolds: MNIST~\citep{LeCun2005TheMD}, Fashion-MNIST (FMNIST)~\citep{xiao2017fashion}, and Kuzushiji-MNIST (KMNIST)~\citep{clanuwat2018deep}, as well as pure structural noise (Uniform and Gaussian). All far off-manifold inputs are upsampled to the $64 \times 64$ grid resolution and normalized to the $[-1, 1]$ interval to ensure they trigger the same magnitude-based activations within the U-Net backbone.

\begin{table}[h]
\centering
\caption{\textbf{Far Off-Manifold Detection Performance (ERA5 Model).} The {\algosmall} model, trained exclusively with local boundary perturbations, successfully generalizes its detection capabilities to deep ambient spaces and extreme semantic shifts, achieving perfect geometric separation across all datasets.}
\label{tab:far_off_manifold}
\small 
\renewcommand{\arraystretch}{0.85} 
\setlength{\tabcolsep}{4pt} 

\begin{tabular*}{\textwidth}{@{\extracolsep{\fill}}lc}
\toprule
\algotext{Far Off-Manifold Condition (Dataset)} & \algotext{AUROC} $\uparrow$ \\ 
\midrule
    \algotext{Gaussian Noise ($\mathcal{N}(0, I)$)} & 1.000 \\
    \algotext{Uniform Noise ($\mathcal{U}[-1, 1]$)} & 1.000 \\
    \algotext{MNIST (Semantic Shift)} & 1.000 \\
    \algotext{FMNIST (Semantic Shift)} & 1.000 \\
    \algotext{KMNIST (Semantic Shift)} & 1.000 \\ 
\bottomrule
\end{tabular*}
\end{table}

As demonstrated in Table~\ref{tab:far_off_manifold}, the model achieves perfect extrapolation detection across all
evaluated regimes. These results confirm that local boundary mining under the {\algo} objective is sufficient to geometrically constrain the entire ambient space. By isolating the valid support, the model guarantees that downstream systems will not receive silently hallucinated physical states when queried with unconstrained inputs.

\section{Split Conformal Prediction Calibration}\label{appendix:split_cp}

To ensure safety-critical reliability during deployment, we employ Split Conformal Prediction \citep{gammerman1998learning, angelopoulos2024theoretical} to construct a statistically valid decision boundary. Because our objective is to detect geometric turbulence, we specifically frame this as a Conformal Anomaly Detection problem. Given a user-specified error rate $\alpha \in (0, 1)$, our goal is to identify a one-sided upper validity threshold $q_{hi}$ such that the probability of a valid, on-manifold input exceeding this threshold is bounded by $\alpha$.

We reserve a calibration dataset of $M$ exchangeable on-manifold conditions. We first compute the divergence score $S_{DOT}$ for each sample and sort them in ascending order, denoting the sorted set as $\mathcal{S}_{cal} = \{s_{(1)}, s_{(2)}, \dots, s_{(M)}\}$. In this conformal anomaly detection setting, lower scores indicate more optimal, valid transport, while higher scores indicate off-manifold extrapolation. Therefore, we allocate our entire error budget $\alpha$ to the upper tail to maximize detection power. We calculate the rank index for the threshold as $k_{hi} = \lceil (M + 1) (1 - \alpha) \rceil$. The decision threshold is then defined by the score at this exact rank: 
\begin{equation}
    q_{hi} = s_{(k_{hi})}
\end{equation}

At inference time, a query condition $c_{test}$ is accepted as valid (in-distribution) if and only if its score does not exceed the calibrated threshold:
\begin{equation}
    S_{DOT}(c_{test}) \le q_{hi}
\end{equation}
If the score falls above this threshold, the generative process is flagged as an unsafe extrapolation. Based on the exchangeability of the calibration and test data, this one-sided procedure provides a rigorous guarantee on the marginal coverage. The probability that a valid on-manifold input is correctly accepted satisfies:
\begin{equation}
    \mathbb{P}\left(S_{DOT}(c_{test}) \le q_{hi}\right) \ge 1 - \alpha
\end{equation}
This ensures that the False Rejection Rate (FRR)---the rate at which valid inputs are incorrectly flagged as unsafe---is controlled exactly at the user-specified level $\alpha$, without needlessly penalizing optimal, low-turbulence trajectories.

\subsection{Interpretation of the FPR Metric}
In our main text experiments, we report the False Positive Rate (FPR) at a fixed calibration error rate of $\alpha = 0.05$ (which corresponds to a Target Coverage of 95\% for valid data). It is critical to distinguish this evaluation metric from the FRR discussed above. 

Because Split Conformal Prediction artificially guarantees the True Positive Rate (TPR) for valid data by forcing it to match the target coverage ($1-\alpha$), the quality of the geometric detection metric is determined entirely by its False Positive Rate---defined here as the fraction of \textit{off-manifold} samples that falsely fall within the acceptable $[q_{lo}, q_{hi}]$ interval.

\begin{itemize}
    \item \textbf{Ideal Behavior:} An ideal geometric detector perfectly separates the score distributions of on- and
        off-manifold inputs. Consequently, even as we lower the threshold to guarantee 99\% coverage of valid data, the
        FPR (extrapolation leakage) should remain near $0\%$.
    \item \textbf{Random Guessing (Baseline):} If the on- and off-manifold score distributions overlap perfectly (as is
        the case with the overly smooth baseline Flow Matching models), accepting $X\%$ of the valid data volume
        statistically necessitates accepting approximately $X\%$ of the invalid volume. This results in a failure mode where
        $\text{FPR} \approx \text{Coverage}$ (e.g., yielding the $\sim 95\%$ FPR observed for unregularized models in
        Table~\ref{tab:2d_results} when targeting $95\%$ coverage).
\end{itemize}

\section{Algorithmic Formulation of {\algo}}
\label{appendix:algorithms}

\begin{algorithm}[h]
   \caption{Training Algorithm of {\algosmall}}
   \label{alg:training}
\begin{algorithmic}[1]
   \Require Dataset $\mathcal{D}$, Loss weights $\lambda, \beta$, Margins $m_r, m_c$
   \State Initialize flow model parameters $\theta$
   \While{not converged}
       \Statex \textit{1. Coupling}
       \State Sample batch $(x_1, c) \sim \mathcal{D}$, prior noise $x_0 \sim p_0$, time $t \sim \mathcal{U}[0,1]$
       \State Construct linear interpolant: $x_t = (1-t)x_0 + tx_1$
       \State Compute target velocity: $u_t = x_1 - x_0$

       \Statex \textit{2. Hard Negative Mining}
       \State $\tilde{c} \gets \textsc{NegativeMining}(c, x_t, u_t, \theta)$ \hfill $\triangleright$ \textit{Agnostic miner (e.g., PGD, Mixup)}

       \Statex \textit{3. Velocity Predictions}
       \State $\hat{v}_{c} \gets v_{t}^{\theta}(x_t, c)$ \hfill $\triangleright$ \textit{On-manifold field}
       \State $\hat{v}_{\tilde{c}} \gets v_{t}^{\theta}(x_t, \tilde{c})$ \hfill $\triangleright$ \textit{Off-manifold field}

       \Statex \textit{4. Objective Computation}
       \State $\mathcal{L}_\algotext{FM} = \|\hat{v}_{c} - u_t\|_2^2$
       \State $\mathcal{L}_\algotext{repel} = \Phi\big(\|u_t - \hat{v}_{c}\|_2 - \|u_t - \hat{v}_{\tilde{c}}\|_2 + m_r\big)$
       \State $\mathcal{L}_\algotext{curve} = \Phi\big(d_{cos}(u_t, \hat{v}_{c}) - d_{cos}(u_t, \hat{v}_{\tilde{c}}) + m_c\big)$
       \State $\mathcal{L}_\algotext{\algosmall} = \mathcal{L}_\algotext{FM} + \lambda \mathcal{L}_\algotext{repel} + \beta
       \mathcal{L}_\algotext{curve}$

       \Statex \textit{5. Optimization Step}
       \State $\theta \gets \theta - \alpha \nabla_\theta \mathcal{L}_{Total}$ \hfill $\triangleright$ \textit{Update via optimizer}
   \EndWhile
\end{algorithmic}
\end{algorithm}

Standard Flow Matching inference involves numerically integrating the learned velocity field $v_t^\theta$ from a prior
noise sample $x_0$ to materialize a data sample $\hat{x}_1$. Algorithm~\ref{alg:inference} details this procedure within
the {\algo} framework. First, we generate the full trajectory $\{\hat{x}_{t_i}\}_{i=1}^N$ using a numerical ODE solver.
Second, we define our constant-velocity geometric anchor $x_t^{\text{ref}}$ as the linear interpolation between the
fixed source $x_0$ and the generated prediction $\hat{x}_1$. As defined in Eq.~\ref{eq:dot_score}, the divergence score
$S_{DOT}$ is computed as the $L_1$ deviation between the realized neural trajectory and this ideal anchor,
spatiotemporally averaged over all dimensions and summed over time steps. 

While the exact $S_{DOT}$ value fundamentally depends on the discrete steps produced by the numerical solver, we observe that the relative separation between on- and off-manifold inputs remains highly robust. As demonstrated in Table~\ref{tab:integrator_ablation}, detection performance (AUROC) and generative quality (LPIPS) remain stable across different solver families (Euler, Heun, RK4) and step counts, indicating that valid and invalid trajectories are affected similarly by discretization.

To formulate a rigorous binary decision during deployment, we compare the calculated divergence score $S_{DOT}$ against
a conformal bound $q_{hi}$, which is pre-calibrated on a hold-out set of valid conditions at a user-specified
significance level $\alpha$. The input condition $c$ is accepted as valid if and only if $S_{DOT} < q_{hi}$; otherwise, the generative process is flagged as an unsafe extrapolation.

\begin{algorithm}[h]
   \caption{Inference and DOT Scoring (Euler Integrator)}
   \label{alg:inference}
\begin{algorithmic}[1]
   \Require Condition $c$, Initial noise $x_0 \sim p_0$, Integration steps $N$
   \Ensure Prediction $\hat{x}_1$, Divergence score $S_{DOT}$

   \Statex \textit{1. Initialization}
   \State $\Delta t \gets 1/N$
   \State $x \gets x_0$
   \State $\mathcal{X}_{traj} \gets [x_0]$ \hfill $\triangleright$ \textit{Store trajectory states}

   \Statex \textit{2. Generative Flow Integration}
   \For{$i = 0$ \textbf{to} $N-1$}
       \State $t \gets i / N$
       \State $v \gets v_{t}^{\theta}(x, c)$ \hfill $\triangleright$ \textit{Evaluate marginal vector field}
       \State $x \gets x + v \cdot \Delta t$ \hfill $\triangleright$ \textit{Euler step}
       \State Append $x$ to $\mathcal{X}_{traj}$
   \EndFor
   \State $\hat{x}_1 \gets x$ \hfill $\triangleright$ \textit{Materialize final prediction}

   \Statex \textit{3. Trajectory Divergence Computation}
   \State $S_{DOT} \gets 0$
   \For{$i = 1$ \textbf{to} $N$} \hfill $\triangleright$ \textit{Evaluate $N$ realized steps}
       \State $t \gets i/N$
       \State $\hat{x}_t \gets \mathcal{X}_{traj}[i]$
       \State $x_t^{\text{ref}} \gets (1-t)x_0 + t\hat{x}_1$ \hfill $\triangleright$ \textit{Compute geometric anchor}
       \State $S_{DOT} \gets S_{DOT} + \frac{1}{D} \|\hat{x}_t - x_t^{\text{ref}}\|_1$ \hfill $\triangleright$ \textit{Accumulate spatial deviation}
   \EndFor
   \State \Return $\hat{x}_1, S_{DOT}$
\end{algorithmic}
\end{algorithm}

\begin{table}[h]
\centering
\caption{\textbf{Evaluation of detection across numerical ODE solvers.} The detection performance (AUROC), mean
    velocity, and sample quality (LPIPS) remain highly stable across various integration schemes and step counts on the
    style transfer task using the PGD mining.}
\label{tab:integrator_ablation}
\small 
\renewcommand{\arraystretch}{0.85} 
\setlength{\tabcolsep}{4pt} 

\begin{tabular*}{\textwidth}{@{\extracolsep{\fill}}llccccc}
\toprule
\multirow{2}{*}{\algotext{Integrator}} & \multirow{2}{*}{\algotext{Steps}} & \multicolumn{2}{c}{\algotext{AUROC}
    $\uparrow$} & \multicolumn{2}{c}{\algotext{Mean Velocity}} & \multirow{2}{*}{\algotext{LPIPS} $\downarrow$} \\
\cmidrule(lr){3-4} \cmidrule(lr){5-6}
& & \algotext{FMNIST} & \algotext{KMNIST} & \algotext{ID} & \algotext{OOD} & \\
\midrule
\algotext{Euler} & 25 & 0.955 & 0.859 & 0.100 & 0.256 & 0.2169 \\
\algotext{Euler} & 50 & 0.954 & 0.855 & 0.097 & 0.278 & 0.2175 \\
\midrule
\algotext{Heun2} & 25 & 0.956 & 0.858 & 0.094 & 0.269 & 0.2198 \\
\algotext{Heun2} & 50 & 0.955 & 0.855 & 0.093 & 0.279 & 0.2196 \\
\midrule
\algotext{RK4}   & 25 & 0.955 & 0.860 & 0.093 & 0.268 & 0.2202 \\
\algotext{RK4}   & 50 & 0.955 & 0.857 & 0.096 & 0.274 & 0.2197 \\
\bottomrule
\end{tabular*}
\end{table}

\section{Ablations and Sensitivity Analysis}\label{appendix:abbl}

\subsection{Efficacy of Negative Sampling Strategies}\label{appendix:mining_abl}

We evaluate {\algosmall} trained with PGD and three alternative negative sampling strategies:
\begin{enumerate}
\item \textbf{Heuristics}: Domain-specific spatial augmentations (random block erasing, elastic deformations, affine shearing, and perspective distortions) intended to create near-manifold perturbations;
    \item \textbf{Masking}: Random and checkerboard masking strategies that destroy local spatial information; and
    \item \textbf{Outlier Exposure}: Using the Omniglot \citep{lake2015human} dataset as a source of explicit, fixed negative samples.
\end{enumerate}

\begin{table}[h]
\centering
\caption{\textbf{Negative Sampling Ablation.} Comparison of boundary-definition strategies. While domain-specific
    heuristics naturally excel on known image manifolds, they fail to generalize to complex physical systems (Weather).
    In the table, we report the detection performance (AUROC$\uparrow$ ).}
\label{tab:unified_ablation}
\small 
\renewcommand{\arraystretch}{0.85} 
\setlength{\tabcolsep}{4pt} 

\begin{tabular*}{\textwidth}{@{\extracolsep{\fill}}lcc}
\toprule
    \algotext{Sampling Method} & \algotext{Temp. Forecast} & \algotext{Style Transfer (FMNIST / KMNIST)} \\
\midrule
\algotext{Masking}            & 0.729 & 0.654 / 0.541 \\
\algotext{Outlier Exposure}   & ---   & 0.710 / 0.572 \\
\algotext{Domain Transforms}  & 0.733 & \textbf{0.987} / \textbf{0.892} \\
\algotext{PGD} & \textbf{0.991} & 0.955 / 0.860 \\
\bottomrule
\end{tabular*}
\end{table}

The results in Table~\ref{tab:unified_ablation} highlight a fundamental limitation of static negative sampling: it
requires \textit{a priori} knowledge of the manifold boundary. \textbf{Heuristic Transforms} perform exceptionally well
on the image tasks (outperforming PGD on FMNIST and KMNIST) because spatial shifts perfectly align with the known
invariances of character recognition. However, this strategy fails completely on the physical Weather forecasting task.
In an atmospheric system, a heuristically transformed pressure field is not simply an "off-manifold" sample; it may represent an entirely different valid state or inherently violate fundamental thermodynamic laws. Without explicit domain knowledge, naive augmentations are unsuitable for defining the boundary of physical validity. 

Similarly, \textbf{Outlier Exposure} and \textbf{Masking} struggle to generalize. Outlier exposure overfits to the explicit geometry of the chosen negative dataset (Omniglot) and fails to detect semantic shifts (KMNIST) that do not structurally resemble those fixed training negatives.

In contrast, \textbf{PGD} serves as a versatile, modality-agnostic baseline. By dynamically optimizing for the most sensitive perturbations in the local neighborhood of the input, PGD automatically probes the relevant boundary constraints. While domain-specific heuristics trivially exploit known symmetries when they are available, PGD provides a reliable fallback that performs well across entirely disparate modalities without requiring manual, domain-specific engineering.

\subsection{Ablation and Sensitivity Analysis on Synthetic Manifolds}\label{appendix:sensitivity_2d}

\paragraph{Ablation Study.} To isolate the contribution of each objective component, we compare {\algosmall} against three
variants: (i) w/o $\mathcal{L}_\algotext{repel}$ ($\lambda=0$) to test directional divergence alone; (ii) w/o
$\mathcal{L}_\algotext{curve}$ ($\beta=0$) to test magnitude penalties alone; and (iii) Random Noise, replacing PGD
negative mining with $\tilde{c} \sim \mathcal{U}[-1, 1]$ to validate the necessity of a more sophisticated
mining strategy.

\begin{table}[h]
\centering
\caption{\textbf{Ablation on Synthetic Manifold Regression.} We isolate the impact of each {\algosmall} component. Replacing PGD with uniform noise ($\tilde{c} \sim \mathcal{U}$) or removing the kinetic repulsion margin ($L_{\algotext{repel}}$) severely degrades detection. Both structural penalties and targeted negative mining are required to minimize leakage (FPR) while maintaining predictive fidelity (MSE).}
\label{tab:abl2d}
\small
\renewcommand{\arraystretch}{0.85} 
\setlength{\tabcolsep}{4pt} 
\begin{tabular*}{\textwidth}{@{\extracolsep{\fill}}llccc}
\toprule
& \algotext{Algorithm} & \algotext{AUROC} $\uparrow$ & \algotext{FPR(\%)} $\downarrow$ & \algotext{MSE} $\downarrow$ \\
\midrule
\multirow{3}{*}{\rotatebox[origin=c]{90}{Reg.}} 
    & \algotext{{\algosmall} w/o $\mathcal{L}_{\algotext{repel}}$}          & 0.711 & 20.45 & \ttp{8.51}{-5} \\
    & \algotext{{\algosmall} w/o $\mathcal{L}_{\algotext{curve}}$}          & 0.931 & 10.69 & \ttp{8.48}{-5} \\
    & \algotext{{\algosmall} w/ $\tilde{c} \sim \mathcal{U}$} & 0.662 & 90.45 & \ttp{8.50}{-5} \\
\bottomrule
\end{tabular*}
\end{table}

The ablation study highlights the critical role of PGD. Replacing PGD with random noise negatives degrades performance
significantly, confirming that easy negatives fail to tighten the vector field around the manifold support. Furthermore,
we observe that the repulsion term $\mathcal{L}_\algotext{repel}$ is the primary driver of the detection signal; removing it
causes a sharper drop in performance compared to removing the curvature penalty $\mathcal{L}_\algotext{curve}$, though both are
required for optimal separation. 

\paragraph{Sensitivity Analysis.} To rigorously evaluate the stability of {\algo}, we conduct a comprehensive hyperparameter sensitivity analysis. We ablate the PGD adversarial radius $\epsilon$, the number of PGD steps $K$, and the contrastive margins ($m_r$, $m_c$) on the 2D Spiral dataset across both the Conditional Generation and Probabilistic Regression tasks. The robust operating ranges and their corresponding performance metrics are summarized in Table~\ref{tab:sensitivity_2d}. 

\begin{table}[h]
\centering
\caption{\textbf{Sensitivity Analysis on the 2D Spiral Manifold.} We report the robust ranges for PGD and contrastive hyperparameters. {\algo} exhibits stability across a wide spectrum of configurations, maintaining near-perfect detection without degrading baseline predictive accuracy (MSE).}
\label{tab:sensitivity_2d}
\small 
\renewcommand{\arraystretch}{0.85} 
\setlength{\tabcolsep}{4pt} 

\begin{tabular*}{\textwidth}{@{\extracolsep{\fill}}llccc}
\toprule
\algotext{Hyperparameter} & \algotext{Tested Range} & \algotext{AUROC} $\uparrow$ & \algotext{FPR (\%)} $\downarrow$ &
    \algotext{MSE ($10^{-5}$)} $\downarrow$ \\
\midrule
\multicolumn{5}{c}{\algotext{Conditional Generation Task}} \\
\midrule
\algotext{Adversarial Radius $\epsilon$} & 0.01 (Fails)   & $\sim$ 0.493 & 95.5 & --- \\
\algotext{Adversarial Radius $\epsilon$} & [0.05, 0.50] & $>$ 0.980    & $<$ 4.9 & --- \\
\algotext{PGD Steps $K$}                 & \{1, 5\}     & $>$ 0.980    & $<$ 4.8 & --- \\
\algotext{Repulsion Margin $m_r$}        & [0.30, 2.00] & $>$ 0.971    & $<$ 5.5 & --- \\
\algotext{Curvature Margin $m_c$}        & [0.01, 1.70] & $>$ 0.979    & $<$ 5.2 & --- \\
\midrule
\multicolumn{5}{c}{\algotext{Probabilistic Regression Task}} \\
\midrule
\algotext{Adversarial Radius $\epsilon$} & [0.01, 0.50] & $>$ 0.983    & $<$ 4.8 & $\le$ 8.52 \\
\algotext{PGD Steps $K$}                 & \{1, 5\}     & $>$ 0.996    & $<$ 2.7 & $\le$ 8.50 \\
\algotext{Repulsion Margin $m_r$}        & [0.50, 2.00] & $>$ 0.993    & $<$ 2.1 & $\le$ 8.53 \\
\algotext{Curvature Margin $m_c$}       & [0.10, 1.70] & $>$ 0.996    & $<$ 3.7 & $\le$ 8.53 \\
\bottomrule
\end{tabular*}
\end{table}

\paragraph{Task-Dependent Geometric Sensitivity.}
The optimal PGD projection radius $\epsilon$ depends directly on the intrinsic thickness (variance) of the underlying data manifold. In the Generation task, the spiral dataset exhibits higher spatial variance. Consequently, a microscopic perturbation ($\epsilon = 0.01$) is insufficient to escape this wider distribution; the perturbed negatives physically overlap with valid data, causing the detection mechanism to fail (AUROC $\sim 0.49$). Conversely, the dataset utilized for the Regression task possesses much lower variance, yielding a significantly "thinner" manifold. Because of this tighter distribution, that exact same $\epsilon = 0.01$ shift easily pushes samples completely off-manifold, resulting in near-perfect separation (AUROC $> 0.98$). In practice, establishing $\epsilon$ based on the natural scale of the normalized domain (e.g., defaulting to $\epsilon = 0.10$) safely clears the boundary across varying manifold thicknesses.

\paragraph{Computational Efficiency and Margin Stability.}
Notably, a single-step PGD attack ($K=1$) performs almost identically to a 5-step attack across both tasks. This proves
that the gradient of the standard Flow Matching loss, when the model is trained with the contrastive objectives, is an strong local signal, allowing {\algo} to
identify hard negatives with minimal computational training overhead. Furthermore, as long as a positive
contrastive margin is enforced, the exact values of $m_r$ and $m_c$ do not require fragile tuning. The detection
mechanism remains highly stable (maintaining AUROC $>0.97$) across broadly swept margin ranges without any major degradation
of the baseline MSE.

\subsection{Sensititivity Analysis on High Dimensional Weather Temperature Forecasting}\label{appendix:sensitivity_era5}

\textbf{Dimensionality and Geometric Sensitivity.} The sensitivity analysis on the high-dimensional ERA5 dataset (Table~\ref{tab:weather_sens}) reveals distinct geometric behaviors compared to the low-dimensional synthetic manifolds. First, we observe that the repulsion margin ($m_r$) must scale significantly. Small margins ($m_r \le 10$) fail to separate from the natural transport energy variance of a $64 \times 64$ system, reducing detection to near-random performance. Robust separation requires a massive kinetic anchor ($m_r \ge 100$) to overcome the high-dimensional smoothness bias. 

Conversely, directional divergence requires a soft penalty. Imposing aggressive curvature bounds ($\beta=5$ or
$m_c=1.8$) degrades the AUROC signal ($\sim 0.75$), whereas gentle angular regularization ($\beta=0.1, m_c=0.3$) yields
near-perfect detection ($0.99+$). In high-dimensional ODE integration, minor initial angular deviations cascade into
massive spatial divergence; over-penalizing this geometry likely disrupts the local Lipschitz continuity required for
stable solving. Finally, across all configurations, the predictive fidelity remains remarkably stable, confirming that
{\algosmall}'s dynamic boundaries govern off-manifold behavior without bleeding into the valid data distribution.

\begin{table}[h]
\centering
    \caption{\textbf{Weather Temperature Forecasting Sensitivity Analysis.} Evaluation of {\algosmall}'s hyperparameters on the high-dimensional ERA5 dataset. Robust detection requires large kinetic margins ($m_r \ge 100$) and tight adversarial boundaries ($\epsilon \le 0.05$), while directional curvature penalties ($\beta, m_c$) must remain gentle to preserve solver stability and maximize the AUROC signal.}
\label{tab:weather_sens}
\small 
\renewcommand{\arraystretch}{0.85} 
\setlength{\tabcolsep}{4pt} 
\begin{tabular*}{\textwidth}{@{\extracolsep{\fill}}lcccc}
\toprule
\algotext{Hyperparameter} & \algotext{MSE ($10^{-4}$)} $\downarrow$ & \algotext{SSIM} $\uparrow$ & \algotext{PSNR}
    $\uparrow$ & \algotext{AUROC} $\uparrow$ \\
\addlinespace
\multicolumn{5}{l}{\color{RoyalBlue}\algotext{\underline{Repulsion Weight}}} \\
\addlinespace
$\lambda = 0.1$        & 7.7 & 0.967 & 37.11 & 0.991 \\
$\lambda = 1.5$        & 7.6 & 0.968 & 37.18 & 0.913 \\
$\lambda = 5.0$        & 7.4 & 0.968 & 37.31 & 0.941 \\
\addlinespace
\multicolumn{5}{l}{\color{RoyalBlue}\algotext{\underline{Repulsion Margin}}} \\
\addlinespace
$m_\text{r} = 1$       & 7.3 & 0.969 & 37.33 & 0.718 \\
$m_\text{r} = 10$      & 7.5 & 0.969 & 37.25 & 0.668 \\
$m_\text{r} = 200$     & 7.4 & 0.968 & 37.27 & 0.949 \\
\addlinespace
\multicolumn{5}{l}{\color{RoyalBlue}\algotext{\underline{Curvature Weight}}} \\
\addlinespace
$\beta = 0.1$          & 7.7 & 0.967 & 37.11 & 0.992 \\
$\beta = 1.5$          & 7.8 & 0.968 & 37.06 & 0.898 \\
$\beta = 5.0$          & 8.8 & 0.966 & 36.93 & 0.765 \\
\addlinespace
\multicolumn{5}{l}{\color{RoyalBlue}\algotext{\underline{Curvature Margin}}} \\
\addlinespace
$m_\text{c} = 0.3$     & 7.3 & 0.969 & 37.34 & 0.977 \\
$m_\text{c} = 1.8$     & 7.7 & 0.966 & 37.11 & 0.752 \\
\addlinespace
\multicolumn{5}{l}{\color{RoyalBlue}\algotext{\underline{Adversarial Radius}}} \\
\addlinespace
\algotext{PGD} $\epsilon = 0.05$  & 8.0 & 0.965 & 36.96 & 0.914 \\
\algotext{PGD} $\epsilon = 0.8$   & 7.0 & 0.969 & 37.37 & 0.640 \\
\bottomrule
\end{tabular*}
\end{table}

\paragraph{Practical Guideline for $\epsilon$ Selection.} The adversarial radius $\epsilon$ governs the distance between
valid and mined negative conditions in the normalized input space. We observe a consistent principle across all tasks:
$\epsilon$ should be set to the natural noise scale of the normalized domain. Concretely, since all inputs are normalized to
$[-1,1]^d$, we recommend an $\epsilon \in [0.07, 0.2]$ as a default. This range safely clears the manifold boundary
across varying manifold thicknesses (as demonstrated in Table~\ref{tab:abl2d} for synthetic data) while remaining in the local neighborhood required for PGD to find informative negatives.
The failure of $\epsilon = 0.8$ on ERA5 is instructive: a perturbation of $0.8$ in the normalized space corresponds to a
physically implausible temperature shift of approximately 83K, which pushes negatives so far off-manifold that they
provide no useful boundary signal. This is analogous to using too large a learning rate — the optimization overshoots
the relevant region.

\section{Visualization of Detection Landscapes}
\label{appendix:2d_landscape}

In this section, we visualize the decision boundaries learned by the models within the 2D ambient conditioning space. Figures~\ref{fig:spiral_gen_landscape} and \ref{fig:spiral_reg_landscape} depict the landscape of accepted versus rejected conditions.

In these plots, the underlying data manifold is the 2D spiral. The background represents the full ambient space $[-1,1]^2$. We evaluate the model on a dense grid of conditions; regions highlighted in \textbf{red} indicate conditions flagged as extrapolations (i.e., inputs that trigger a high $S_\text{DOT}$ score).

The results visually confirm the silent failure hypothesis: Standard Flow Matching (left) accepts almost the entire
ambient space as valid, driven by its inherent smoothness bias. In contrast, {\algosmall} (right) effectively creates a tight
validity boundary that follows the data support, while rejecting the surrounding off-manifold regions (red).

\begin{figure}[h]
    \centering
    \begin{subfigure}[b]{0.28\textwidth}
        \centering
        \includegraphics[width=\textwidth]{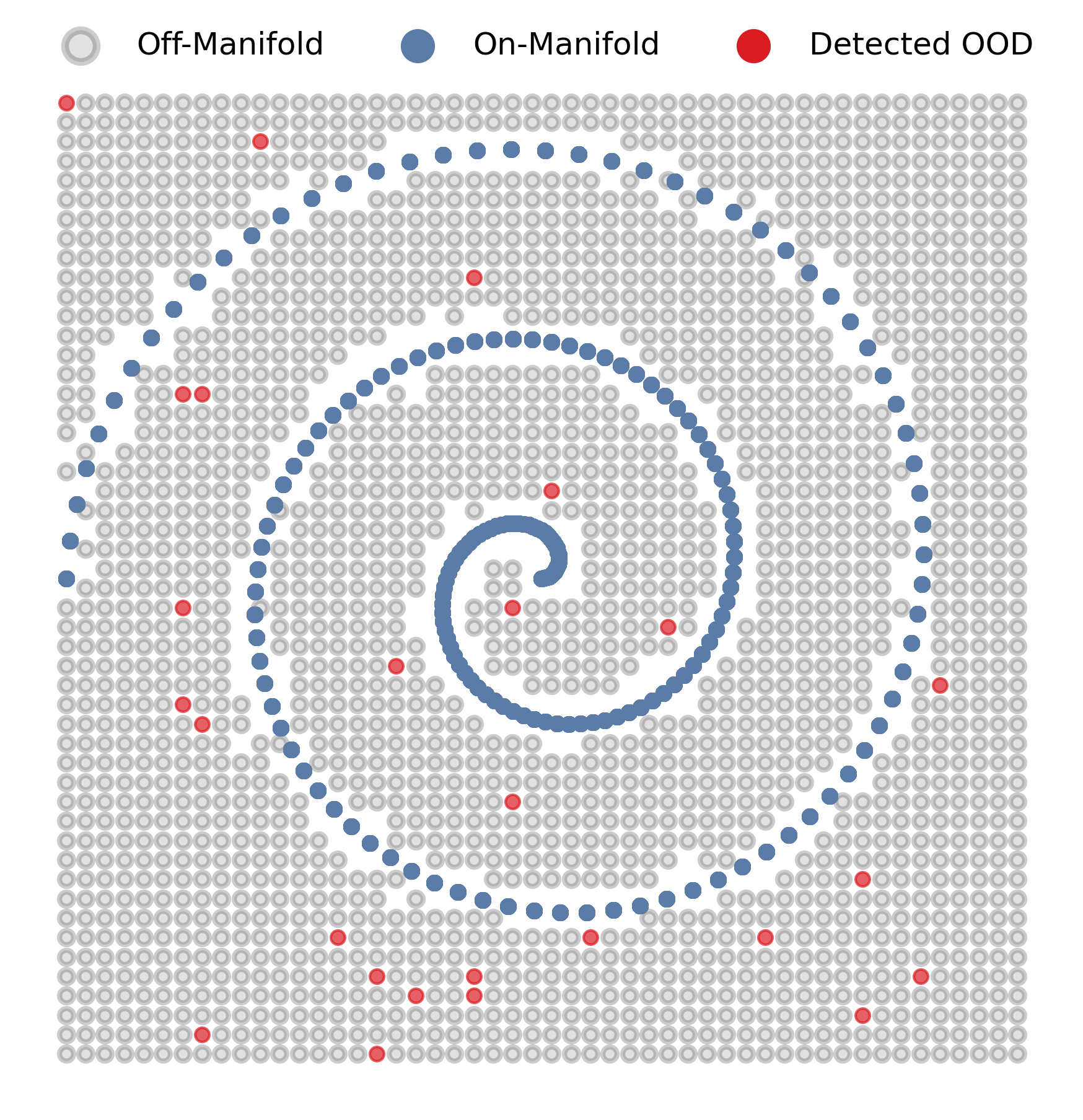}
        \caption{Flow Matching}
    \end{subfigure}
    \begin{subfigure}[b]{0.28\textwidth}
        \centering
        \includegraphics[width=\textwidth]{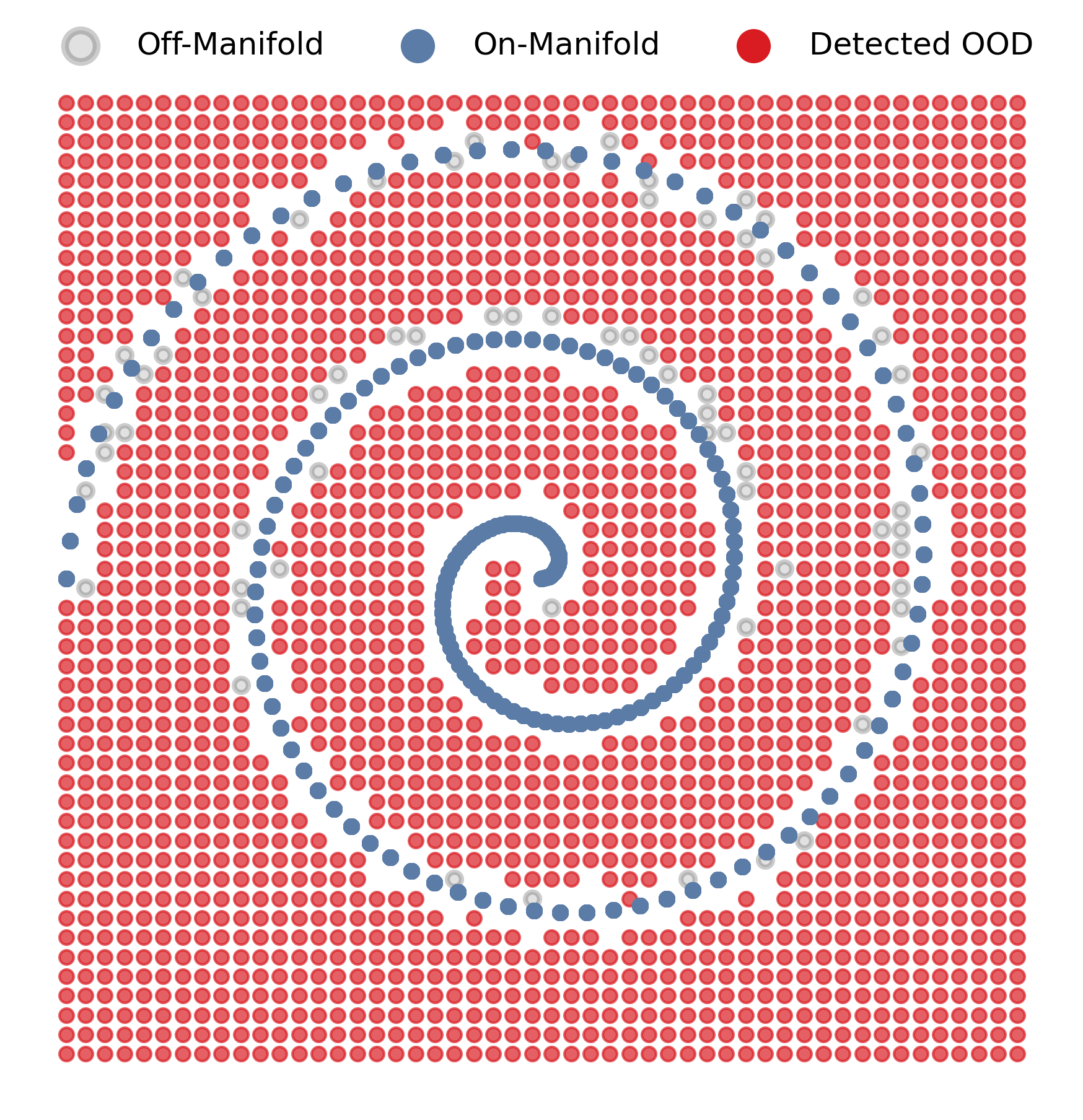}
        \caption{{\algo} (Ours)}
    \end{subfigure}
    \caption{\textbf{Detection Landscape:} Conditional Generation. The standard model (a) fails to distinguish the manifold from the background. {\algo} (b) successfully identifies the ambient space as invalid (red), accepting only the spiral support.}
    \label{fig:spiral_gen_landscape}
\end{figure}

\begin{figure}[h]
    \centering
    \begin{subfigure}[b]{0.28\textwidth}
        \centering
        \includegraphics[width=\textwidth]{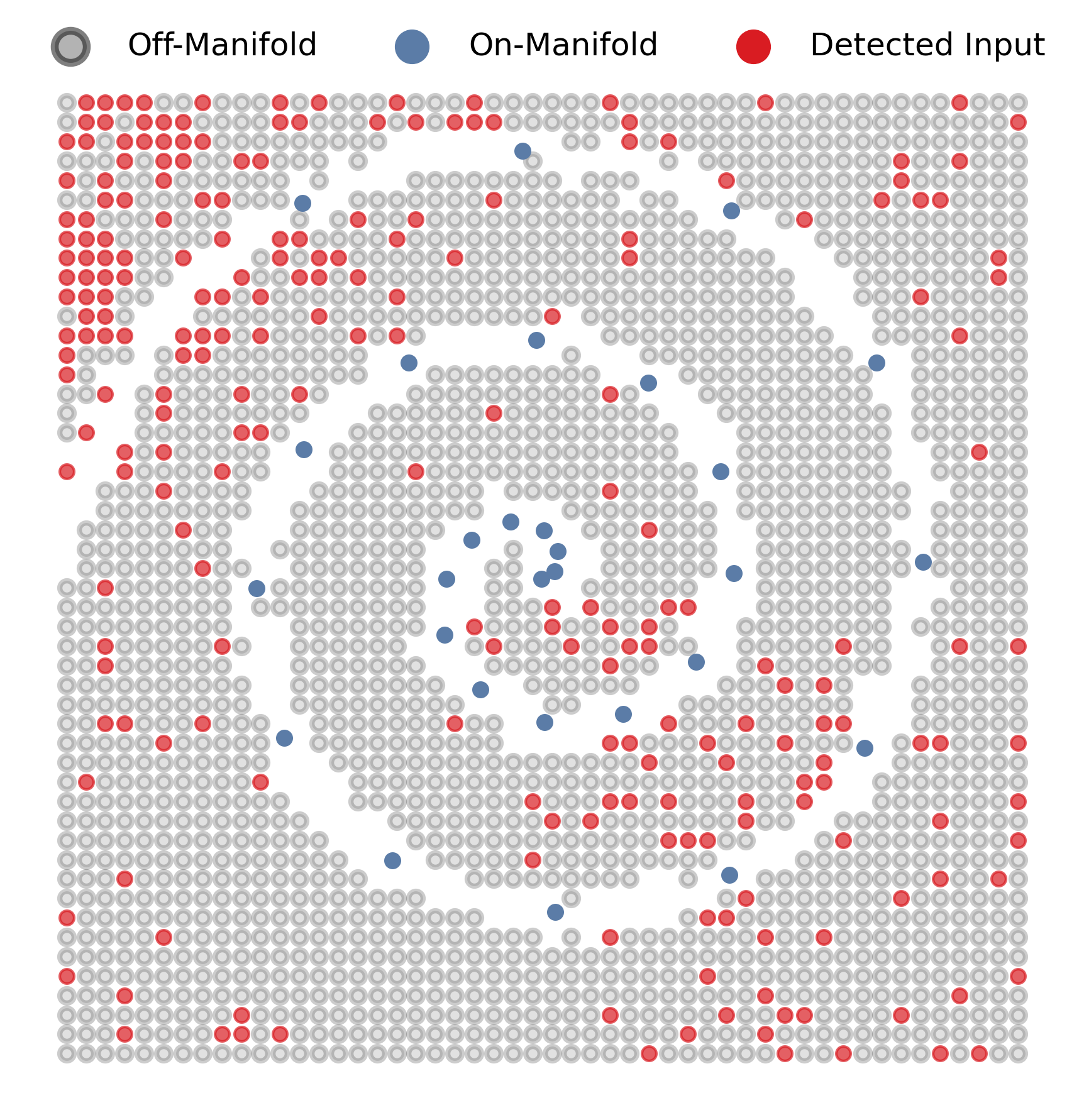}
        \caption{Flow Matching}
    \end{subfigure}
    \begin{subfigure}[b]{0.28\textwidth}
        \centering
        \includegraphics[width=\textwidth]{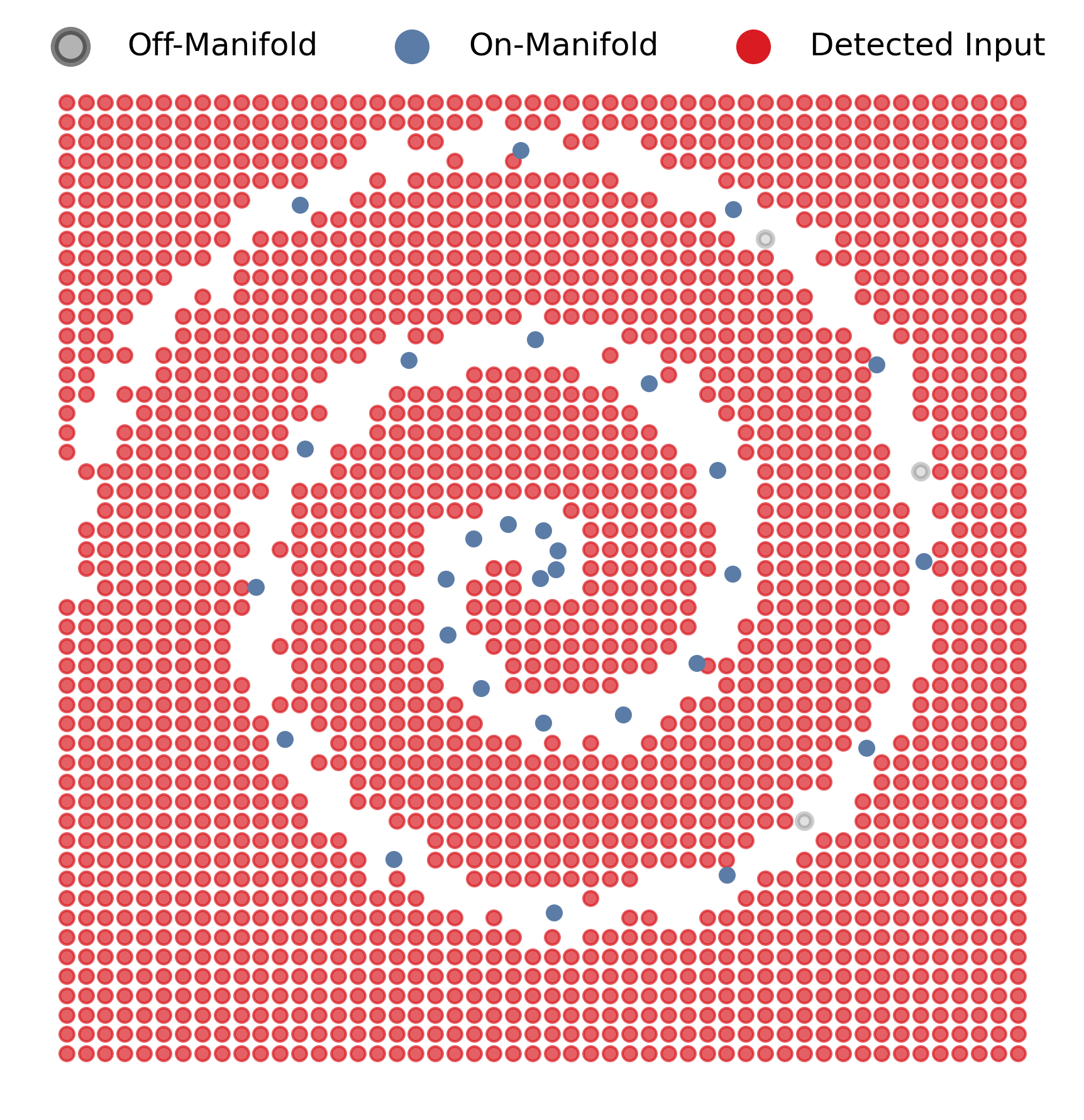}
        \caption{{\algo} (Ours)}
    \end{subfigure}
    \caption{\textbf{Detection Landscape:} Probabilistic Regression. Even in the regression task, where the condition represents a dynamic state $x_t$, {\algo} (b) maintains a precise decision boundary around the valid trajectory path.}
    \label{fig:spiral_reg_landscape}
\end{figure}

\section{Comparison with DiffPath}
\label{appendix:diffpath}

For completeness, we provide a reference comparison against DiffPath~\citep{heng2024out}, a method that analyzes reverse generative trajectories for OOD detection. Unlike {\algo}, which actively regularizes the underlying vector field during training, DiffPath operates as a passive, post-hoc detector on a frozen model.

Because DiffPath requires a pre-trained generative model, and no off-the-shelf models exist for our specific scientific
datasets, we trained the diffusion models from scratch on our exact in-distribution data. DiffPath utilizes this model to perform the reverse generation process on a hold-out set of ID calibration samples. It extracts various discrete trajectory statistics and fits a statistical density estimator (e.g., Gaussian Mixture Model) to these features. At inference time, the validity of a query input is determined by the likelihood of its trajectory statistics under this fitted distribution.

Table~\ref{tab:diffpath_ablation} presents the comparative detection performance. While DiffPath-6D achieves marginally higher AUROC scores, it relies on the assumption that the frozen model will implicitly produce distinguishable trajectory features for OOD inputs. In contrast, {\algo} embeds safety directly into the dynamics, mathematically enforcing geometric divergence. We observe that our active training approach yields detection capabilities highly comparable to this expensive post-hoc analysis.

\begin{table}[h!]
\centering
\caption{\textbf{Comparison with DiffPath.} We report detection performance (AUROC) against DiffPath, a post-hoc
    trajectory auditor. {\algosmall} (active regularization) achieves highly competitive detection capabilities directly within the generative dynamics.}
\label{tab:diffpath_ablation}
\small 
\renewcommand{\arraystretch}{0.85} 
\setlength{\tabcolsep}{4pt} 

\begin{tabular*}{\textwidth}{@{\extracolsep{\fill}}lcc}
\toprule
\algotext{Method} & \algotext{Temp. Forecast} $\uparrow$ & \algotext{Style Transfer (FMNIST / KMNIST)} $\uparrow$ \\
\midrule
\multicolumn{3}{l}{\color{RoyalBlue}{\algotext{Passive Post-Hoc Statistics}}} \\
    \algotext{DiffPath-1D} & 0.973 & 0.981 / 0.863 \\
    \algotext{DiffPath-6D} & 0.993 & 0.991 / 0.929 \\
\midrule
\multicolumn{3}{l}{\color{RoyalBlue}{\algotext{Active Generative Regularization}}} \\
\algotext{\algosmall (Transforms)} & 0.733 & 0.987 / 0.892 \\
\algotext{\algosmall (PGD)} & 0.991 & 0.955 / 0.860 \\
\bottomrule
\end{tabular*}
\end{table}

\paragraph{Remark on Flow vs. Diffusion Dynamics.} It is critical to note that DiffPath is designed for Diffusion Models, which solve a stochastic differential equation (SDE) or its curved probability flow ODE equivalent. In that regime, the learned vector fields naturally exhibit complex, non-linear trajectories where OOD inputs \emph{might} implicitly induce detectable erratic behavior or jitter. However, deploying such passive detectors on Flow Matching models is fundamentally challenging. As established in Section~\ref{sec:extrapolation_problem}, standard Optimal Transport Flow Matching explicitly minimizes transport energy, forcing valid trajectories to be straight-line geodesics. Due to the strong Lipschitz continuity of the neural network, unregularized FM models collapse OOD trajectories into efficient, straight paths that are geometrically indistinguishable from valid data. Consequently, passive monitoring of path statistics is frequently insufficient for Flow Models, necessitating the active geometric regularization introduced in {\algo} to explicitly shatter this smoothness bias and enforce divergence.

\section{Experiment Details}\label{appendix:exp_details}

\subsection{Theoretical Formulation of Detection Baselines}
\label{appendix:baselines}

We compare {\algo} against two intrinsic baselines: (i) a likelihood proxy 
derived from the learned vector field, and (ii) ensemble predictive variance.

\subsubsection{Flow-Matching Likelihood}

A natural extrapolation signal is whether the model assigns low marginal probability 
to the queried condition. Although Flow Matching optimizes a transport-based objective 
rather than a direct likelihood bound, the learned vector field $v_t^\theta$ governs a 
probability flow ODE that induces a tractable density over the state space.

Let the state $x_t$ evolve according to $\frac{dx_t}{dt} = v_t^\theta(x_t, c)$. By the 
continuity equation, the log-density evolves along the trajectory as \cite{lipman2024flow}:
\begin{equation}
    \frac{d}{dt} \log p_t(x_t|c) = -\nabla_x \cdot v_t^\theta(x_t, c)
\end{equation}
Since the divergence of the vector field equals the trace of its Jacobian, integrating 
from $t=0$ to $t=1$ yields the implied log-likelihood:
\begin{equation}
    \log p_1(x_1|c) = \log p_0(x_0) - \int_0^1 \text{Tr}\left(\nabla_x v_t^\theta(x_t, 
    c)\right) dt
\end{equation}
Computing the exact Jacobian trace requires $\mathcal{O}(d)$ backward passes, so we 
instead approximate it with the Hutchinson estimator. Drawing a probe vector 
$z \sim \mathcal{N}(0, I)$ gives an unbiased stochastic estimate:
\begin{equation}
    \log p_1(x_1|c) \approx \log p_0(x_0) - \int_0^1 z^\top \nabla_x v_t^\theta(x_t, 
    c)\, z \, dt
\end{equation}
The Jacobian-vector product $z^\top \nabla_x v_t^\theta$ is computed via reverse-mode 
automatic differentiation, integrated jointly with the state ODE. The extrapolation score 
is then defined as the negative proxy log-density, $S_{\text{div}}(c) = -\log p_1(x_1|c)$, 
where higher values indicate low compatibility under the model's learned marginal dynamics.

\subsubsection{Ensemble Predictive Variance}

We additionally evaluate Deep Ensembles as a standard epistemic uncertainty baseline. 
We independently train $K=3$ Flow Matching models $\{\theta_k\}_{k=1}^K$ under identical 
hyperparameters but distinct random initializations.

To decouple epistemic uncertainty from aleatoric stochasticity, we evaluate each model's 
deterministic conditional expectation $\hat{x}_1^{(k)}$ under condition $c$, and define 
the extrapolation score as the empirical spatial variance across predictions:
\begin{equation}
    S_{\text{Ens}}(c) = \frac{1}{K} \sum_{k=1}^K \|\hat{x}_1^{(k)} - \mu_c\|_2^2, 
    \quad \mu_c = \frac{1}{K} \sum_{k=1}^K \hat{x}_1^{(k)}
\end{equation}
For valid conditions, the independently trained models converge to consistent predictive 
flows, yielding low variance. For off-manifold queries, the models disagree, inflating 
$S_{\text{Ens}}$. The critical limitation of this approach is computational: detection 
requires $K$ independent ODE solves at inference time, imposing an $\mathcal{O}(K)$ 
overhead that is prohibitive in real-time deployments.
\subsection{General Experimental Setup}
Across all evaluated tasks, experiments were conducted using a single NVIDIA RTX 4090 GPU. During optimization, we applied a learning rate schedule consisting of a 500-step linear warmup from $10^{-8}$ to the base learning rate, followed by a cosine decay. Prior to training, all continuous input data across tasks was normalized to the $[-1, 1]$ range to ensure stable vector field regression and consistent application of the $L_\infty$ adversarial bounds. To ensure statistically rigorous evaluation and properly calibrated detection boundaries, all datasets were partitioned using a fixed, unified split ratio: 70\% utilized for model training, 10\% reserved as a hold-out calibration set for Split Conformal Prediction, and 20\% utilized for final testing, if a test set is not provided.

\subsection{Synthetic Manifold Experiment Details}
\label{appendix:synthetic}

\textbf{Data Generation.} We utilize the 2D spiral and ambient OOD setup as formally defined in
Section~\ref{sec:spiral_sec}. We enforce a distance of 0.025 between the manifold support and the OOD sampling region to ensure the detection task is unambiguous.

\textbf{Model Architecture.} We employ a time-conditioned Multi-Layer Perceptron (MLP) that processes inputs as flat vectors. Specifically, the \textbf{input} is a direct concatenation of the noisy state ($x_t \in \mathbb{R}^2$ in the conditional generation task or $x_t \in \mathbb{R}^{20}$ in the regression task), the condition $c \in \mathbb{R}^2$, and the time scalar $t \in [0,1]$. The \textbf{backbone} consists of three hidden layers of width 512 using SiLU activations, followed by a linear \textbf{output} projection to $\mathbb{R}^2$ predicting the continuous velocity field $u_t$.

\textbf{Training Setup.} The model is trained using the AdamW optimizer for 400,000 training steps in the generation task and 200,000 training steps in the regression task. Detailed hyperparameters for the architecture and the contrastive loss functions are listed in Table~\ref{tab:synthetic_hyper_combined}.

\begin{table}[h]
\centering
\caption{\textbf{Hyperparameters for Synthetic Manifold Experiments.}}
\label{tab:synthetic_hyper_combined}
\small 
\renewcommand{\arraystretch}{0.85} 
\setlength{\tabcolsep}{4pt} 

\begin{tabular*}{\textwidth}{@{\extracolsep{\fill}}lcc}
\toprule
\algotext{Hyperparameter} & \algotext{Conditional Generation} & \algotext{Probabilistic Regression} \\
\midrule
\algotext{Hidden Layers} & 3 & 3 \\
\algotext{Hidden Dimension} & 512 & 512 \\
\algotext{Activation} & \algotext{SiLU} & \algotext{SiLU} \\
\midrule
\algotext{Optimizer} & \algotext{AdamW} & \algotext{AdamW} \\
\algotext{Learning Rate} & \ttp{3}{-4} & \ttp{3}{-4} \\
\algotext{Weight Decay} & \ttp{1}{-3} & \ttp{1}{-3} \\
\algotext{Batch Size} & 256 & 256 \\
\algotext{Training Steps} & 400,000 & 200,000 \\
\midrule
\algotext{PGD Steps ($K$)} & 3 & 3 \\
\algotext{Perturbation ($\epsilon$)} & 0.1 & 0.1 \\
\algotext{Repulsion Weight ($\lambda$)} & 0.1 & 0.1 \\
\algotext{Curvature Weight ($\beta$)} & 0.05 & 0.1 \\
\algotext{Margins ($m_r$ / $m_c$)} & 1.0 / 0.9 & 1.0 / 0.9 \\
\midrule
\algotext{Integrator} & \algotext{Euler} & \algotext{Euler} \\
\algotext{Integration Steps} & 50 & 50 \\
\algotext{Significance Level} ($\alpha$) & 0.05 & 0.05 \\
\algotext{Target Coverage} & 95\%  & 95\% \\
\bottomrule
\end{tabular*}
\end{table}

\paragraph{Note on the Spiral Benchmark.} The 2D spiral is a canonical benchmark in the study of Neural ODEs and continuous-time dynamical systems \citep{chen2018neural}. Unlike simple shapes, a spiral represents a stiff geometric structure where the optimal transport path (a straight line) often intersects regions that are off-manifold. In the context of probabilistic regression, this experiment is critical because it validates that the learned vector field captures the true non-linear differential equation governing the system's evolution, rather than memorizing discrete point-to-point mappings. Successfully modeling this dynamics while simultaneously detecting off-manifold perturbations demonstrates that {\algo} can enforce safety constraints without over-simplifying the complex curvature of the underlying data manifold.

\subsection{Weather Temperature Forecasting}\label{appendix:weather}

\textbf{Dataset.} We demonstrate the applicability of {\algo} to climate science via two-meter surface temperature
forecasting using the ERA5 reanalysis dataset \citep{era5paper}. To formulate the predictive task, we construct
conditional pairs $(c, x_1)$, where the condition $c$ represents the historical temperature map at time $t$ and the
target $x_1$ is the realized state at $t + 6$ hours. The data is sampled at standard six-hour synoptic intervals (00,
06, 12, 18 UTC). Physically, the global temperature limits (spanning roughly $210$K to $313$K) are normalized
to the $[-1, 1]$ interval.

\textbf{Model Architecture.} We employ a U-Net backbone modified for conditional generation. For \textbf{time conditioning}, the time step $t \in [0,1]$ is embedded using Fourier features followed by an MLP to produce a 128-dimensional embedding, which is injected into every residual block via FiLM layers. The \textbf{condition encoder} processes the conditioning input $c$ (current state) through a separate learnable network consisting of a convolution, two residual blocks, and one downsampling operation. Finally, for \textbf{cross-attention injection}, the latent representation from the condition encoder is projected to match the time embedding dimension and injected into the U-Net backbone via cross-attention layers at the $16\times16$ resolution (lowest level).

\textbf{Training Setup.} We train the model for 50,000 training steps with a batch size of 64. To ensure stability, we use gradient clipping (norm 1.0) and an Exponential Moving Average (EMA) of model weights with decay $0.9999$.

\begin{table}[h]
\centering
\caption{\textbf{Hyperparameters used in the Weather Temperature Forecasting (ERA5) experiment.}}
\label{tab:weather_hyper}
\small 
\renewcommand{\arraystretch}{0.85} 
\setlength{\tabcolsep}{4pt} 

\begin{tabular*}{\textwidth}{@{\extracolsep{\fill}}lc}
\toprule
\algotext{Hyperparameter} & \algotext{Value} \\
\midrule
\algotext{Input Resolution} & $64 \times 64$ \\
\algotext{Channels} & [64, 64, 128] \\
\algotext{ResBlocks per Scale} & 2 \\
\algotext{Attention Resolution} & $16 \times 16$ \\
\algotext{Conditioning} & \algotext{Cross-Attention} \\
\algotext{Time Embedding Dim} & 128 \\
\midrule
\algotext{Optimizer} & \algotext{AdamW} \\
\algotext{Learning Rate} & \ttp{3}{-4} \\
\algotext{Weight Decay} & \ttp{1}{-2} \\
\algotext{AdamW Betas} & [0.9, 0.95] \\
\algotext{Batch Size} & 64 \\
\algotext{Training Steps} & 50,000 \\
\algotext{EMA Decay} & 0.9999 \\
\algotext{Gradient Clip} & 1.0 \\
\midrule
\algotext{PGD Steps ($K$)} & 5 \\
\algotext{Perturbation ($\epsilon$)} & 0.1 \\
\algotext{Repulsion Weight ($\lambda$)} & 0.7 \\
\algotext{Curvature Weight ($\beta$)} & 0.9 \\
\algotext{Margin $m_\text{r}$} & 100.0 \\
\algotext{Margin $m_\text{c}$} & 0.9 \\
\midrule
\algotext{Integrator} & \algotext{Euler} \\
\algotext{Integration Steps ($N$)} & 50 \\
\algotext{Significance Level ($\alpha$)} & 0.05 \\
\algotext{Target Coverage} & 95\% \\
\bottomrule
\end{tabular*}
\end{table}

\subsection{Cross-Domain Style Transfer Details}\label{appendix:style_transfer}

\textbf{Dataset.} We perform a cross-domain translation task mapping grayscale MNIST digits ($1\times32\times32$) to RGB SVHN digits ($3\times32\times32$). 

\textbf{Model Architecture.} We use a U-Net backbone similar to the weather temperature forecasting experiment, but with a stronger dual-conditioning mechanism to preserve high-frequency structural details. For \textbf{input concatenation}, we explicitly concatenate the 1-channel conditioning image $c$ with the 3-channel noisy state $x_t$ at the model input, providing strong spatial guidance for the digit structure. Additionally, $c$ is processed by a shallow condition encoder (1 ResBlock) and injected via \textbf{cross-attention} at the $16\times16$ and $8\times8$ resolutions. The \textbf{backbone} utilizes channel widths of [64, 64, 128] with attention layers at the two lower resolutions.

\textbf{Training \& Inference.} The model is trained for 500,000 training steps using AdamW with a learning rate decay
schedule. For inference, we perform $N=50$ steps.

\begin{table}[h]
\centering
\caption{\textbf{Hyperparameters used in the Style Transfer (MNIST $\to$ SVHN) experiment.}}
\label{tab:style_hyper}
\small 
\renewcommand{\arraystretch}{0.85} 
\setlength{\tabcolsep}{4pt} 

\begin{tabular*}{\textwidth}{@{\extracolsep{\fill}}lc}
\toprule
\algotext{Hyperparameter} & \algotext{Value} \\
\midrule
\algotext{Input Resolution} & $32 \times 32$ \\
\algotext{Channels} & [64, 64, 128] \\
\algotext{ResBlocks per Scale} & 2 \\
\algotext{Attention Resolutions} & $16 \times 16, 8 \times 8$ \\
\algotext{Conditioning} & \algotext{Concat + Cross-Attention} \\
\midrule
\algotext{Optimizer} & \algotext{AdamW} \\
\algotext{Learning Rate} & \ttp{1}{-4} \\
\algotext{Weight Decay} & \ttp{1}{-3} \\
\algotext{AdamW Betas} & [0.9, 0.95] \\
\algotext{Accumulation Steps} & 4 \\
\algotext{Training Steps} & 500,000 \\
\algotext{Batch Size} & 128 \\
\algotext{EMA Decay} & 0.9999 \\
\midrule
\algotext{PGD Steps ($K$)} & 5 \\
\algotext{Perturbation ($\epsilon$)} & 0.2 \\
\algotext{Repulsion Weight ($\lambda$)} & 0.1 \\
\algotext{Curvature Weight ($\beta$)} & 0.7 \\
\algotext{Margin $m_\text{r}$} & 10.0 \\
\algotext{Margin $m_\text{c}$} & 1.3 \\
\midrule
\algotext{Integrator} & \algotext{Euler} \\
\algotext{Integration Steps ($N$)} & 50 \\
\algotext{Significance Level ($\alpha$)} & 0.05 \\
\algotext{Target Coverage} & 95\% \\
\bottomrule
\end{tabular*}
\end{table}

\section{Additional Qualitative Results}
\label{app:qualitative}

In this section, we provide weather temperature forecasting (Figure ~\ref{fig:app_weather}) and style transfer (Figure ~\ref{fig:app_style}) samples from {\algo}. 

\begin{figure}[h!]
    \begin{center}
        \includegraphics[width=0.80\textwidth]{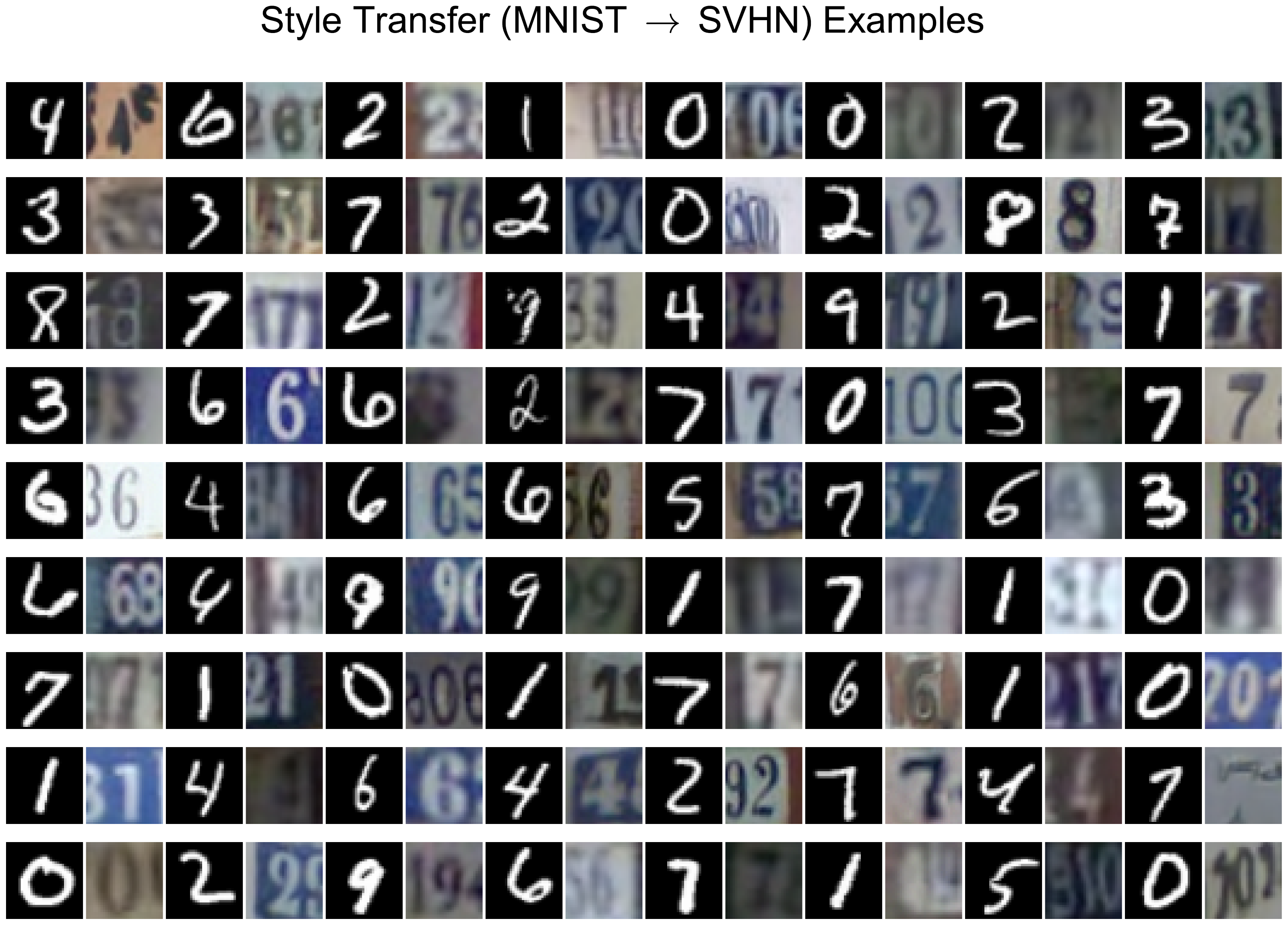}
    \end{center}
    \caption{\textbf{Extended Style Transfer Samples.} Pairs of Conditioning Inputs (MNIST) and Generated Outputs
    (SVHN). The model successfully maps the sparse structural guidance of the grayscale digit to the rich, multi-modal
    distribution of street view imagery.}
    \label{fig:app_style}
\end{figure}

\begin{figure}[H]
    \begin{center}
        \includegraphics[width=0.95\textwidth]{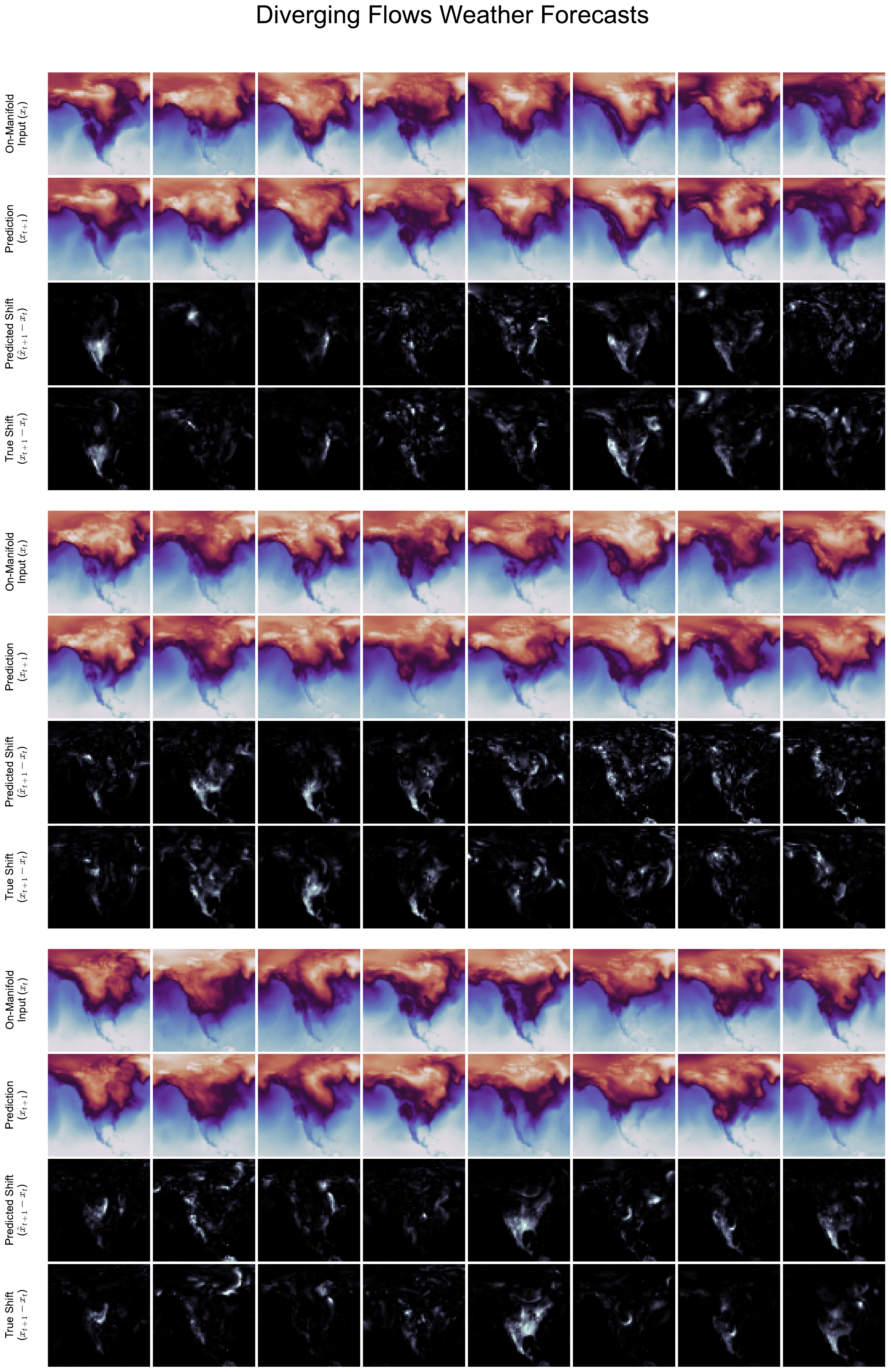}
    \end{center}
    \caption{\textbf{Extended Weather Temperature Forecasting Samples.} Randomly selected 6-hour temperature forecasts generated by {\algo} on the ERA5 validation set. The model generates physically consistent heatmaps that align with the ground truth dynamics.}
    \label{fig:app_weather}
\end{figure}



\end{document}